\documentclass[a4paper,fleqn]{cas-dc}

\usepackage[numbers]{natbib}
\usepackage{amsthm}
\newtheorem{proposition}{Proposition}
\usepackage{algorithm}
\usepackage{algpseudocode}
\usepackage{hyperref}
\usepackage{caption}
\usepackage{xcolor}
\usepackage{adjustbox}
\def\tsc#1{\csdef{#1}{\textsc{\lowercase{#1}}\xspace}}
\tsc{WGM}
\tsc{QE}
\tsc{EP}
\tsc{PMS}
\tsc{BEC}
\tsc{DE}

\begin{document}
\let\WriteBookmarks\relax
\def\floatpagepagefraction{1}
\def\textpagefraction{.001}
\shorttitle{Cluster-based Low-Rank Matrix Approximation for Medical Image Compression}
\shortauthors{S. Hamlomo and M. Atemkeng}

\title [mode = title]{Clustering-Based Low-Rank Matrix Approximation for Medical Image Compression}                      



\author[1,2]{Sisipho Hamlomo}[type=editor,
                        orcid=0000-0003-0744-8433]
\cormark[1]
\fnmark[1]
\ead{s.hamlomo@ru.ac.za}
\credit{Conceptualization, Methodology, Software, Formal analysis, Investigation, Writing - Original Draft}

\affiliation[1]{organization={Department of Statistics, Rhodes University},
                addressline={PO Box 94}, 
                city={Makhanda},
                postcode={6140}, 
                state={Eastern Cape},
                country={South Africa}}
                
\author[2,3]{Marcellin Atemkeng}[%
   ]
\fnmark[2]
\ead{m.atemkeng@ru.ac.za}

\credit{Conceptualization, Methodology, Formal analysis, Writing - Review \& Editing, Supervision, Funding acquisition}

\affiliation[2]{organization={Department of Mathematics, Rhodes University},
                addressline={PO Box 94}, 
                city={Makhanda},
                postcode={6140}, 
                state={Eastern Cape},
                country={South Africa}}

\affiliation[3]{organization={National Institute for Theoretical and Computational Sciences (NITheCS)},
                city={Stellenbosch},
                postcode={7600}, 
                state={Western Cape}, 
                country={South Africa}}

\cortext[cor1]{Corresponding author}

\begin{abstract}
 Medical images are inherently high-resolution and contain locally varying structures that are crucial for diagnosis. Efficient compression of such data must therefore preserve diagnostic fidelity while minimizing redundancy. Low-rank matrix approximation (LoRMA) techniques have shown strong potential for image compression by capturing global correlations; however, they often fail to adapt to local structural variations across regions of interest, i.e., tumor areas or other diagnostically significant regions. To address these limitations, we introduce an adaptive LoRMA,  which  partitions a medical image into overlapping patches, groups structurally similar patches into several clusters using k-means, and performs SVD within each cluster.  We derive the overall compression factor accounting for patch overlap and analyze how patch size influences compression efficiency and computational cost. While the proposed adaptive LoRMA method is applicable to any data exhibiting high local variation, we focus on medical imaging due to its pronounced local variability. We evaluate and compare our adaptive LoRMA against global SVD across four imaging modalities: MRI, ultrasound, CT scan, and chest X-ray. Results demonstrate that adaptive LoRMA effectively preserves structural integrity, edge details, and diagnostic relevance, as measured by peak signal-to-noise ratio (PSNR), structural similarity index (SSIM), mean squared error (MSE), intersection over union (IoU), and edge preservation index (EPI). Adaptive LoRMA significantly minimizes block artifacts and residual errors, particularly in pathological regions, consistently outperforming global SVD in terms of PSNR, SSIM, IoU, EPI, and achieving lower MSE. Adaptive LoRMA prioritizes clinically salient regions  while allowing aggressive compression in non-critical regions, optimizing storage efficiency.  Although adaptive LoRMA requires higher processing time, its diagnostic fidelity justifies the overhead for high-compression applications.
\end{abstract}




\maketitle

\section{Introduction}
Modern medical imaging modalities such as MRI, CT, ultrasound, and X-ray produce high-resolution images that enable precise clinical diagnosis but pose significant challenges in terms of storage, transmission, and computational efficiency \cite{lingala2011accelerated, zhou2014low, saha2023matrix}. Low-rank matrix approximation (LoRMA) has become a fundamental tool in machine learning and other domains due to its ability to capture essential structures in high-resolution and large-scale data while suppressing noise and redundancy. By representing a high-resolution medical image with a lower-rank approximation, LoRMA reduces storage requirements and computational costs while preserving essential structural information. A widely used approach to LoRMA is the singular value decomposition (SVD) \cite{kishore2017literature}, which provides an optimal low-rank approximation under the Frobenius norm, as demonstrated by the Eckart-Young theorem \cite{eckart1936approximation}. However, despite its effectiveness in many domains, the global nature of truncated SVD imposes limitations when applied to real-world data, particularly when local variations are critical. In medical imaging, for example, many medical images exhibit gradual decay in their singular values, making it challenging to choose a hard threshold for truncation without risking the loss of important features \cite{hansen1990discrete, leblond2010singular, candes2011robust}. Truncated SVD (what we refer in this paper as global SVD) applies uniform compression across the entire medical image, disregarding local variations in structural complexity. This uniformity assumes homogeneity in the data's structure, an assumption that rarely holds true for many medical images \cite{smith2021variability, brooks2013quantification, hadjidemetriou2009restoration}. For example, while some regions of a medical image may be relatively smooth, others can contain important features. The global SVD technique not only fails to adapt to these local variations but can also introduce over-smoothing and block artifacts, leading to the loss of important features \cite{gonzalez2009digital, leblond2010singular, golub2013matrix, gu2014weighted, aishwarya2016lossy}. Furthermore, the combinatorial nature of rank minimization makes strict rank constraints computationally intractable for large-scale data matrices \cite{recht2010guaranteed, candes2012exact}.

Local low-rank matrix approximation (LLoRMA) was proposed to address these limitations. The LLoRMA technique, initially introduced by Lee \cite{lee2013local} in 2013, was applied to recommendation systems to improve prediction accuracy by capturing local structures within data matrices. Their experiments utilized datasets such as MovieLens, EachMovie, and Netflix, focusing on improving user preference predictions. While their primary focus was on recommendation systems, subsequent research has extended the application of LLoRMA to other domains, such as hyperspectral imaging \cite{wang2017hyperspectral, he2019non},  dynamic MRI reconstruction \cite{otazo2015low, yaman2017locally}, and so on. Other studies, such as non-local means denoising \cite{coupe2006fast, manjon2008mri, manjon2010adaptive, buades2011non, dutta2013non} and block-matching 3D filtering \cite{dabov2006image, chen2010image, feruglio2010block, eksioglu2016decoupled}, demonstrated the benefits of patch-based processing for preserving and extracting important image features. Building on these ideas, this paper proposes a cluster-based SVD mathematical framework. The cluster-based SVD technique partitions a medical image into overlapping patches, groups similar patches using k-means, and applies SVD within each cluster. By adapting the compression factor based on the local medical image features, the method aims to preserve and extract important features while reducing storage burdens. 

Our key contributions are summarized as follows:
\begin{itemize}
    \item We theoretically demonstrate that a global rank \(r_g\) does not always exist such that cluster-based and global SVD compression factors are approximately the same. This arises due to structural differences in how information redundancy and local variations are exploited by each method.
    \item We provide a comprehensive derivation of the overall compression factor, explicitly incorporating redundancy introduced by overlapping patches. 
    \item We analytically show that the overall compression factor increases monotonically with patch size. This observation shows the importance of patch size selection in balancing compression factor and computational cost.
    \item We introduce a pseudo-code to calculate the overlap proportion and determine the effective number of patches. 
    \item We discuss the computational cost of cluster-based SVD in detail and how it can be reduced.
    \item We evaluate cluster-based and global SVD methods across different medical image modalities, such as MRI, CT scan, ultrasound, and X-ray. Our experiments provide insights into modality-specific performance and demonstrate the advantages of cluster-based SVD in preserving and extracting important features.
\end{itemize}

The rest of this paper is organized as follows: Section~\ref{sec:background} introduces low-rank matrix approximation and its challenges; Section~\ref{sec:method} describes our proposed cluster-based SVD method in detail. Section~\ref{sec:evaluation} introduces the evaluation metrics used to assess the performance of our proposed cluster-based SVD method. Section~\ref{sec:experiment} discusses the experimental setup and the results. Section~\ref{sec:discussion} discusses the findings and limitations, and Section~\ref{sec:conclusion} concludes with key insights and future directions.\section{Technical background}{\label{sec:background}}
\subsection{Low-rank matrix approximation}
Many real-world data matrices exhibit redundancy \cite{saha2023matrix}, meaning their information can be represented using a lower-dimensional subspace without significant loss. This property motivates the use of low-rank approximation, where a medical image is approximated by another medical image of reduced rank, preserving essential structures and extracting important features while discarding noise and redundant information. A fundamental problem in low-rank approximation is to find a rank-\( r_g \) matrix \( \mathbf{X} \) that best approximates a given medical image \( \mathbf{A} \in \mathbb{R}^{m \times n} \) under the Frobenius norm
\begin{align}
\min_{\mathbf{X} \in \mathcal{M}_{r_g}} \| \mathbf{A} - \mathbf{X} \|_F^2,
\end{align}
where the feasible set is defined as 
\begin{align}
\mathcal{M}_{r_g} = \{\mathbf{X} \in \mathbb{R}^{m \times n} : \operatorname{rank}(\mathbf{X}) = r_g \}. 
\end{align}
By the Eckart–Young theorem, the optimal solution to this problem is given by the truncated SVD. This choice of \( \mathbf{A}_{r_g} \) minimizes the approximation error. That is,  
\begin{align}
\mathbf{A}_{r_g} = \underset{\mathbf{X} \in \mathcal{M}_{r_g}}{\operatorname{argmin}} \|\mathbf{A} - \mathbf{X}\|_F^2.
\end{align}

In many practical scenarios, the singular values of \( \mathbf{A} \) decay gradually, and imposing a hard rank threshold may result in the loss of important structural components \cite{candes2012exact}. To address these limitations, an alternative approach is to replace the nonconvex rank constraint with its convex envelope, the nuclear norm \cite{recht2010guaranteed}. The nuclear norm encourages low-rank structure while keeping the optimization problem convex. The relaxed optimization problem is then formulated as
\begin{align}
\min_{\mathbf{X} \in \mathbb{R}^{m \times n}} \|\mathbf{A} - \mathbf{X}\|_F^2 + \lambda \|\mathbf{X}\|_*,
\end{align}
where \(\|\mathbf{X}\|_* = \sum_{i=1}^{\min(m,n)} \sigma_i(\mathbf{X})\) denotes the nuclear norm, \(\sigma_i(\mathbf{X})\) denotes the singular values of \(\mathbf{X}\) and \( \lambda > 0 \) is a regularization parameter controlling the trade-off between reconstruction fidelity and rank penalization. This formulation is particularly useful when dealing with noisy or incomplete data, as it allows for a more stable approximation. A widely used method for solving nuclear norm minimization problems is the singular value thresholding (SVT) algorithm \cite{cai2010singular}, which iteratively updates \( \mathbf{X} \) using the rule
\begin{align}
\mathbf{X}^{(k+1)} = D_{\tau} \Bigl( \mathbf{X}^{(k)} + \mu \bigl(\mathbf{A} - \mathbf{X}^{(k)}\bigr) \Bigr),
\end{align}
where \( \mu \) is a step size, \( \tau = \lambda\mu \), and the soft-thresholding operator \( D_{\tau}(\mathbf{Y}) \) is applied to the singular values as follows:
\begin{align}
D_{\tau}(\mathbf{Y}) = \mathbf{U} S_{\tau}(\mathbf{\Sigma}) \mathbf{V}^T, \quad \text{where} \quad S_{\tau}(\sigma_i) = \max\{\sigma_i - \tau, 0\}.
\end{align}
This iterative approach shrinks small singular values, effectively reducing rank while maintaining significant structural information in \( \mathbf{A} \). An alternative constrained formulation replaces the nuclear norm penalty with an explicit bound
\begin{align}
\min_{\mathbf{X} \in \mathbb{R}^{m \times n}} \|\mathbf{A} - \mathbf{X}\|_F^2 \quad \text{s.t.} \quad \|\mathbf{X}\|_* \leq \tau.
\end{align}
This formulation ensures a direct control over the nuclear norm of \( \mathbf{X} \), thereby influencing its rank. Duality between the penalized and constrained formulations allows for algorithmic flexibility, with methods such as the alternating direction method of multipliers frequently used for efficient optimization \cite{eckstein2015understanding, lin2022alternating, wei2012distributed}. However, different singular values contribute unequally to the intrinsic structure of the data. The standard nuclear norm penalizes all singular values uniformly, which can lead to over-shrinking of significant components. To overcome this issue, \cite{gu2014weighted} proposed weighted nuclear norm minimization (WNNM) which introduces weights \( w_i > 0 \) to form a weighted nuclear norm
\begin{align}
\|\mathbf{X}\|_{w,*} = \sum_{i=1}^{\min(m,n)} w_i \sigma_i(\mathbf{X}),
\end{align}
thereby allowing larger singular values, which often capture essential structural information, to be penalized less. The resulting optimization problem is then written as
\begin{align}
\min_{\mathbf{X} \in \mathbb{R}^{m \times n}} \|\mathbf{A} - \mathbf{X}\|_F^2 + \lambda \|\mathbf{X}\|_{w,*},
\end{align}
or in a constrained form
\begin{align}
\min_{\mathbf{X} \in \mathbb{R}^{m \times n}} \|\mathbf{A} - \mathbf{X}\|_F^2 \quad \text{s.t.} \quad \|\mathbf{X}\|_{w,*} \leq \tau.
\end{align}
Solving the WNNM problem typically involves a weighted singular value thresholding operator, where each singular value \( \sigma_i \) is shrunk by an amount proportional to its corresponding weight, i.e., 
\begin{align}
S_{\tau_i}(\sigma_i) = \max(\sigma_i - \lambda w_i, 0).
\end{align}
This nonuniform shrinkage better preserves important features in \( \mathbf{A} \) while effectively reducing noise.

There exist several low-rank approximation methods that have not been discussed above. For example, robust principal component analysis decomposes a matrix into a low-rank component and a sparse component, making it particularly effective in scenarios with significant outliers or structured noise \cite{zhao2014robust, bao2012inductive, liu2019adaptive}. In addition, CUR decomposition approximates a matrix by selecting a subset of its actual columns and rows, providing a more interpretable decomposition that preserves key data characteristics \cite{drineas2008relative, mahoney2009cur, boutsidis2014optimal}. Other methods, such as QR decomposition with column pivoting and low-rank matrix factorization, offer computationally efficient alternatives for obtaining low-rank approximations, with the former prioritizing the most informative columns and the latter being widely used in matrix completion tasks \cite{duersch2017randomized, xiao2017fast}. While these techniques provide valuable alternatives for various applications, in this paper we are only concerned about SVD that is why it warrants a discussion.
\subsection{Global SVD}
\subsubsection{Theoretical overview}
SVD decomposes a medical image, $\mathbf{A}\in \mathbb{R}^{m\times n}$ ($m>n$), into three independent matrices:\begin{align}{\label{SVD1}}
\mathbf{A}&=\mathbf{U}\mathbf{\Sigma} \mathbf{V}^T\\
&=\sum_{i=1}^n\sigma_i \mathbf{u}_i\mathbf{v}_i^T,
\end{align}
where $\mathbf{U}\in \mathbb{R}^{m\times m}$ and $\mathbf{V}\in \mathbb{R}^{n \times n}$ are unitary matrices, $\mathbf{\Sigma}$ is a diagonal matrix and $^T$ denotes the transpose. The diagonal entries $\sigma_i$ of $\mathbf{\Sigma}$ are uniquely determined by $\mathbf{A}$ and are known as the singular values of $\mathbf{A}$. The singular values are arranged in descending order, that is, $\sigma_1\geq \sigma_2\geq \sigma_3\cdots \geq\sigma_n$. The columns of $\mathbf{U}$ and the columns of $\mathbf{V}$ are the left and right singular vectors of $\mathbf{A}$ respectively. The vectors $\mathbf{u}_i \in \mathbb{R}^m$ and $\mathbf{v}_i\in \mathbb{R}^n$ are the orthonormal eigenvectors of $\mathbf{AA}^T$ and  $\mathbf{A}^T\mathbf{A}$ respectively. By retaining only $r_g$ (where $1<r_g<n$) singular values, we obtain a compressed medical image, $\mathbf{A}_{r_g}$, which is an approximation of $\mathbf{A}$. That is, 
\begin{gather}
    \mathbf{A}_{r_g}=\mathbf{U}_{r_g}\mathbf{\Sigma}_{r_g}\mathbf{V}^T_{r_g}\approx \mathbf{A},
\end{gather}
where $\mathbf{U}_{r_g}$ denotes a matrix of size $m\times r_g$ containing the first $r_g$ columns of $\mathbf{U}$, $\mathbf{\Sigma}_{r_g}$ denotes a diagonal matrix of size $r_g\times r_g$ containing $r_g$ singular values, and $\mathbf{V}_{r_g}^T$ denotes a matrix of size $r_g\times n$ containing the first $r_g$ columns of $\mathbf{V}$. Each value of $\sigma_i$ quantifies how much the corresponding component $\sigma_i \mathbf{u}_i\mathbf{v}_i^T$ contributes in the reconstruction of $\mathbf{A}$, that is, the larger the value of $\sigma_i$ the more $\sigma_i \mathbf{u}_i\mathbf{v}_i^T$ contributes to the reconstruction of $\mathbf{A}$. The computational cost for reconstructing \(\mathbf{A}\) is approximately given by
\begin{align}{\label{globalcost}}
\text{cost}_{\text{global SVD}}\sim\mathcal{O}(mnr_g).
\end{align}
For large data matrices, the computational cost in Eq.~\ref{globalcost} can scale poorly.
\subsubsection{Data matrix compression}
The compression factor, $CF_{\text{global}}(r_g)$, is defined as the ratio between the size of the uncompressed data and the size of the compressed data. For an $m\times n$ matrix the number of sub-matrices needed to compute $\mathbf{A}_{r_g}$ is $r_g(m+n+1)$ \cite{atemkeng2023lossy}. This leads to an overall $CF_{\text{global}}(r_g)$ of 
\begin{align}{\label{components}}
    CF_{\text{global}}(r_g)=\frac{mn}{r_g(m+n+1)}.
\end{align}
Using Eq.~\ref{components}, the number of singular values to retain is computed as
\begin{align}
    r_g=\left\lceil\frac{mn}{CF_{\text{global}}(r_g)(m+n+1)}\right\rceil,
\end{align}
where $\lceil\cdot\rceil$ gives the smallest integer. The value of $r_g$ directly influences the degree of compression achieved for a medical image. If fewer singular values are retained, more signal is discarded, resulting in higher compression but possibly lower matrix quality. Conversely, retaining more singular values preserves more signal, resulting in higher approximated matrix quality but less compression. The compression loss can be computed as 
\begin{align}{\label{loss}}
    \rVert\mathbf{A}-\mathbf{A}_{r_g}\lVert_F\leq \epsilon \rVert\mathbf{A}\lVert_F,
\end{align}
where $\epsilon$ denotes a tolerance parameter determining the maximum allowable difference between the original matrix $\mathbf{A}$ and its approximated version $\mathbf{A}_{r_g}$. The minimum compression loss resulting from the Frobenius norm of the zeroed singular values is \cite{eckart1936approximation}
\begin{align}
    \rVert\mathbf{A}-\mathbf{A}_{r_g}\lVert_F=\sqrt{\sum\limits_{k=r_g+1}^{n}\sigma^2_k}.
\end{align}
The number of singular values to retain can be chosen such that
\begin{align}{\label{retained components}}
&\lVert\mathbf{A}_{r_g}\rVert_F\geq \alpha \lVert\mathbf{A}\rVert_F\\
        & \sqrt{\sum\limits_{k=1}^{r_g}\sigma_{k}^2}\geq \alpha \sqrt{\sum\limits_{k=1}^{n}\sigma_{k}^2}\label{mintolarance},
\end{align}
where $\alpha$ denotes the minimum percentage of the signal to be preserved. That is 
\begin{align}{\label{threshold}}
r_g = \operatorname*{argmin}_{r'} \left\{ r' : \sum_{k=1}^{r'} \sigma_k^2 \ge \alpha \sum_{k=1}^{n} \sigma_k^2 \right\}.
\end{align}
Eq.~\ref{threshold} emphasizes that $r_g$ is the smallest value that ensures the retained signal energy meets the specified threshold, $\alpha$.
\section{Method}{\label{sec:method}}
\subsection{Overall workflow}
The proposed method exploits the local and non-local self-similarity of medical images by first grouping similar patches using k-means and then applying SVD within each cluster. This enables adaptive compression while preserving and extracting important features. In contrast to global SVD (shown in Fig.~\ref{fig:workflow}(a)), which compresses the entire matrix uniformly, our approach enables localized compression based on structural complexity. Fig.~\ref{fig:workflow}(b) illustrates the main steps of our proposed cluster-based SVD method. The process starts with patch extraction, where the input medical image is divided into overlapping patches. Next, the patches are grouped using k-means clustering to organize patches with similar structural patterns into different clusters. Each cluster is then processed independently with truncated SVD, where the number of singular values retained is adapted to the local complexity of the cluster. The compressed cluster matrices are then aggregated and reconstructed to obtain the final compressed matrix. This ensures that only the most informative components are retained, while redundant and less important pixels are removed.
\begin{figure*}
    \centering
    \includegraphics[width=0.9\textwidth]{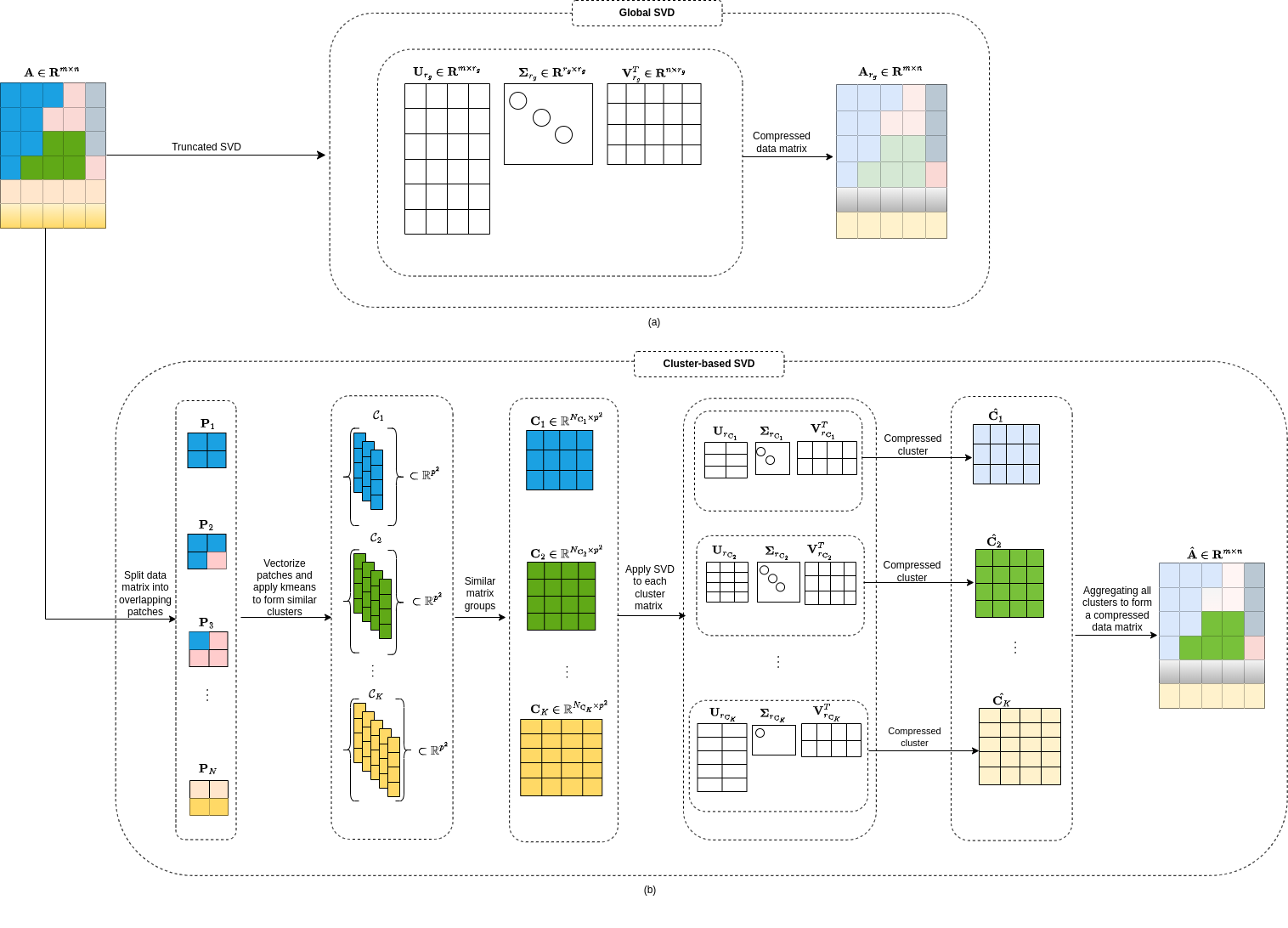}%
    \caption{Overview of the cluster-based and global SVD for data compression. (a) Global SVD applies truncated SVD directly to the original medical image $\mathbf{A}$, yielding low-rank components $\mathbf{U}_{r_g}$, $\boldsymbol{\Sigma}_{r_g}$, and $\mathbf{V}_{r_g}^T$, from which the compressed medical image $\mathbf{A}_{r_g}$ is reconstructed. (b) In cluster-based SVD, the original medical image \(\mathbf{A}\) is first split into overlapping \(p\times p\) patches \(\{\mathbf{P}_i\}_{i=1}^N\). Each patch is vectorized into \(\mathbb{R}^{p^2}\) and grouped into \(K\) clusters \(\{\mathcal{C}_k\}_{k=1}^K\) using k‑means, which are used to form cluster matrices \(\mathbf{C}_k\in\mathbb R^{N_{\mathbf{C}_k}\times p^2}\). SVD is then applied independently to each \(\mathbf{C}_k\) to produce compressed cluster matrices \(\widehat{\mathbf{C}}_k\). Finally, the compressed patches are aggregated to reconstruct the compressed medical image \(\widehat{\mathbf{A}}\).}
    \label{fig:workflow}
\end{figure*}
\subsection{Patch grouping}
In this paper, we use the k-means clustering algorithm to group similar patches, thereby exploiting the local and non-local self-similarity of the medical images \cite{1467423, dabov2007image}. The idea is to identify and group patches that have similar features, even if they are spatially distant from each other. The k-means algorithm is used for proof of concept in this paper and other similar measure methods could also be considered; however this is not the aim for this paper, and we leave it for future work. To group similar patches using k-means clustering, we divide the medical image into overlapping patches (this helps to capture more contextual information and reduce the boundary artifacts) each of size \( p \times p \) (as shown in Algorithm~\ref{alg:patch_grouping}, line 5). Each patch \( \boldsymbol{P}_k \) is vectorized, that is,
\begin{align}
\mathbf{p}_k = \text{vec}(\boldsymbol{P}_k) \in \mathbb{R}^{p^2}.
\end{align}
We partition the patches into \( K \) clusters, \(\boldsymbol{\mathcal{C}}= \{ \mathcal{C}_1, \mathcal{C}_2, \cdots, \mathcal{C}_K \} \), where each cluster \( \mathcal{C}_i \) is a set of vectors such that
\begin{align}
\mathcal{C}_i \cap \mathcal{C}_j = \emptyset \quad \text{for all } i \neq j, \quad \bigcup_{i=1}^K \mathcal{C}_i = \{\mathbf{p}_k\}_{k=1}^N,
\end{align}
where \(N\) denotes the total number of patches. This ensures that each patch \( \mathbf{p}_k \) is assigned to only one cluster, minimizing the within-cluster variance (for more details, see Proposition~\ref{min variance}). At each iteration of k-means, the patches are assigned to the cluster with the closest centroid 
\begin{align}
\mathcal{C}_i = \left\{ \mathbf{p}_k : i = \arg\min_j \|\mathbf{p}_k - \boldsymbol{\mu}_j\|^2 \right\},
\end{align}
where \( \boldsymbol{\mu}_j \in \mathbb{R}^{p^2} \) is the centroid of cluster \( \mathcal{C}_j \). The goal is to minimize the following objective function
\begin{align}{\label{within-cluster variance}}
J(\boldsymbol{\mathcal{C}}, \{\boldsymbol{\mu}_i\}) = \sum_{i=1}^K \sum_{\mathbf{p}_k \in \mathcal{C}_i} \|\mathbf{p}_k - \boldsymbol{\mu}_i\|^2.
\end{align}
The centroids \( \boldsymbol{\mu}_i \) are updated after each iteration as
\begin{align}
\boldsymbol{\mu}_i = \frac{1}{N_{\mathbf{C}_i}} \sum_{\mathbf{p}_k \in \mathcal{C}_i} \mathbf{p}_k,
\end{align}
where \(N_{\mathbf{C}_i}\) denotes the number of patches in cluster \(\mathcal{C}_i \). This iterative process continues until convergence, where the cluster assignments no longer change.
\begin{proposition}{\label{min variance}}
The within-cluster variance  
\begin{align*}
J(\boldsymbol{\mathcal{C}}, \{\boldsymbol{\mu}_i\}) = \sum_{i=1}^{K} \sum_{\mathbf{p}_k \in \mathcal{C}_i} \|\mathbf{p}_k - \boldsymbol{\mu}_i\|^2
\end{align*}
is minimized if and only if (1) for each \(i\), \(\boldsymbol{\mu}_i\) is the arithmetic mean of the vectors in \(\mathcal{C}_i\), and (2) each \(\mathbf{p}_k\) is assigned to the cluster whose centroid is nearest in Euclidean distance.
\end{proposition}
\begin{proof}[\textbf{Proof}]\mbox{}\\
\textbf{Case 1:} Minimization with respect to the centroid for fixed cluster \(\mathcal{C}_i\).\\[1mm]
Consider a cluster \(\mathcal{C}_i \subset \{\mathbf{p}_k\}_{k=1}^{N}\) and let \(\mathbf{c} \in \mathbb{R}^{p^2}\) be an arbitrary candidate for its center. Define the function  
\begin{align*}
f_i(\mathbf{c}) = \sum_{\mathbf{p}_k \in \mathcal{C}_i} \|\mathbf{p}_k - \mathbf{c}\|^2.
\end{align*}
Expanding the squared norm yields  
\begin{align*}
\|\mathbf{p}_k - \mathbf{c}\|^2 = \langle \mathbf{p}_k - \mathbf{c},\, \mathbf{p}_k - \mathbf{c} \rangle = \|\mathbf{p}_k\|^2 - 2\langle \mathbf{p}_k, \mathbf{c} \rangle + \|\mathbf{c}\|^2,
\end{align*}
where \(\langle \cdot \rangle\) is the inner product. It then follows that
\begin{align*}
f_i(\mathbf{c}) = \sum_{\mathbf{p}_k \in \mathcal{C}_i} \|\mathbf{p}_k\|^2 - 2 \left\langle \sum_{\mathbf{p}_k \in \mathcal{C}_i} \mathbf{p}_k, \mathbf{c} \right\rangle + \sum_{\mathbf{p}_k \in \mathcal{C}_i} \|\mathbf{c}\|^2.
\end{align*}
Since \(\|\mathbf{c}\|^2\) is independent of the summation index and there are \(N_{\mathbf{C}_i}\) elements in \(\mathcal{C}_i\), this becomes  
\begin{align*}
f_i(\mathbf{c}) = \sum_{\mathbf{p}_k \in \mathcal{C}_i} \|\mathbf{p}_k\|^2 - 2 \left\langle \sum_{\mathbf{p}_k \in \mathcal{C}_i} \mathbf{p}_k, \mathbf{c} \right\rangle + N_{\mathbf{C}_i}\|\mathbf{c}\|^2.
\end{align*}
The term \(\sum_{\mathbf{p}_k \in \mathcal{C}_i} \|\mathbf{p}_k\|^2\) does not depend on \(\mathbf{c}\). Hence, to minimize \(f_i(\mathbf{c})\), it suffices to minimize  
\begin{align*}
g_i(\mathbf{c}) = -2 \left\langle \sum_{\mathbf{p}_k \in \mathcal{C}_i} \mathbf{p}_k, \mathbf{c} \right\rangle + N_{\mathbf{C}_i}\, \|\mathbf{c}\|^2.
\end{align*}
Taking the gradient with respect to \(\mathbf{c}\), we have  
\begin{align*}
\nabla_{\mathbf{c}} g_i(\mathbf{c}) = -2 \sum_{\mathbf{p}_k \in \mathcal{C}_i} \mathbf{p}_k + 2N_{\mathbf{C}_i}\, \mathbf{c}.
\end{align*}
Setting the gradient equal to the zero vector and solving for \(\mathbf{c}\) yields  
\begin{align*}
\mathbf{c} = \frac{1}{N_{\mathbf{C}_i}} \sum_{\mathbf{p}_k \in \mathcal{C}_i} \mathbf{p}_k.
\end{align*}
Hence, the optimal centroid for cluster \(\mathcal{C}_i\) is  
\begin{align*}
\boldsymbol{\mu}_i = \frac{1}{N_{\mathbf{C}_i}} \sum_{\mathbf{p}_k \in \mathcal{C}_i} \mathbf{p}_k.
\end{align*}
\\
\textbf{Case 2:} Optimality of the assignment.\\[1mm]
Assume that the centroids \(\{\boldsymbol{\mu}_i\}_{i=1}^{K}\) are fixed as above. For each patch \(\mathbf{p}_k\), its contribution to the overall variance \(J\) is
\begin{align*}
\|\mathbf{p}_k - \boldsymbol{\mu}_i\|^2 \quad \text{if} \quad \mathbf{p}_k \in \mathcal{C}_i.
\end{align*}
To minimize \(J\), for every patch \(\mathbf{p}_k\) we must choose the cluster \(C_i\) such that
\begin{align*}
\|\mathbf{p}_k - \boldsymbol{\mu}_i\| \le \|\mathbf{p}_k - \boldsymbol{\mu}_j\| \quad \text{for all } j \neq i.
\end{align*}
That is, each patch \(\mathbf{p}_k\) should be assigned to the cluster whose centroid is nearest in the Euclidean sense. If a patch were assigned to any other cluster, its squared distance (and hence its contribution to \(J\)) would be larger, contradicting the minimality of \(J\).
\end{proof}
Algorithm~\ref{alg:patch_grouping} describes how to extract overlapping patches from a medical image and group similar patches using k-means clustering. It iteratively assigns patches to the nearest centroid and updates centroids until convergence.
\begin{algorithm}[!htbp]
\caption{Grouping patches using K-means}\label{alg:patch_grouping}
\begin{algorithmic}[1]
\State \textbf{Input:} medical image $\mathbf{A} \in \mathbb{R}^{m \times n}$, patch size $p$, number of clusters $K$
\State \textbf{Output:} Clusters $\{\mathcal{C}_1, \mathcal{C}_2, \dots, \mathcal{C}_K\}$ of vectorized patches
\State $\mathcal{P} \gets \emptyset$
\For{each valid top-left coordinate $(i,j)$ in $\mathbf{A}$}
    \State $\mathbf{P}_k \gets \mathbf{A}[i : i+p-1,\; j : j+p-1]$
    \State $\mathbf{p}_k \gets \operatorname{vec}(\mathbf{P}_k) \in \mathbb{R}^{p^2}$
    \State $\mathcal{P} \gets \mathcal{P} \cup \{\mathbf{p}_k\}$
\EndFor
\State Initialize centroids $\{\boldsymbol{\mu}_1, \boldsymbol{\mu}_2, \dots, \boldsymbol{\mu}_K\}$ randomly chosen from $\mathcal{P}$.
\Repeat
    \For{each patch \(\mathbf{p}_k \in \mathcal{P}\)}
        \State Compute distances: \(d(j) = \|\mathbf{p}_k - \boldsymbol{\mu}_j\|_2\) for \(j=1,\dots,K\) \Comment{\(d(j)\) is the Euclidean distance between patch \(\mathbf{p}_k\) and centroid \(\boldsymbol{\mu}_j\)}
        \State Let \(i = \operatorname{argmin}_{1 \le j \le K}\, d(j)\)
        \State Assign \(\mathbf{p}_k\) to cluster \(\mathcal{C}_i\)
    \EndFor
    \For{each cluster \(\mathcal{C}_i\), \(i=1,\dots,K\)}
        \State Update centroid: 
        \[
        \boldsymbol{\mu}_i \gets \frac{1}{N_{\mathbf{C}_i}} \sum_{\mathbf{p}_k \in \mathcal{C}_i} \mathbf{p}_k
        \]
    \EndFor
\Until{the cluster assignments do not change}
\State \Return \(\{\mathcal{C}_1, \mathcal{C}_2, \dots, \mathcal{C}_K\}\)
\end{algorithmic}
\end{algorithm}
\subsection{Cluster-dependent SVD}
Once the patches have been grouped into clusters, each cluster is represented as a matrix. For a cluster \(\mathcal{C}_k\) containing \(N_{\mathbf{C}_k}\) patches, each represented by the vector \(\mathbf{p}_i \in \mathbb{R}^{p^2}\), the corresponding medical image representation is given by
\begin{align}
\mathbf{C}_k = \left(\begin{array}{cccc} \mathbf{p}_1^T & \mathbf{p}_2^T & \cdots & \mathbf{p}_{N_{\mathbf{C}_k}}^T \end{array}\right)^T \in \mathbb{R}^{N_{\mathbf{C}_k} \times p^2},
\end{align}
where each row corresponds to a vectorized patch. The aggregated cluster matrix is then constructed by stacking all cluster matrices. That is,
\begin{align}
\mathbf{C} =  \left(\begin{array}{cccc} \mathbf{C}_1 & \mathbf{C}_2 & \cdots & \mathbf{C}_K \end{array}\right)^T \in \mathbb{R}^{N \times p^2}.
\end{align}
Let \(\Gamma\) be the SVD operator for any matrix \(\mathbf{X}\), defined as
\begin{align}
\Gamma(\mathbf{X}) = \{\mathbf{U}_{r_g},\mathbf{\Sigma}_{r_g},\mathbf{V}_{r_g}\} \quad \text{such that} \quad \mathbf{X} = \mathbf{U}_{r_g}\,\mathbf{\Sigma}_{r_g}\,\mathbf{V}_{r_g}^T.
\end{align}
In our framework, \(\Gamma\) is applied independently on each cluster matrix, so that the aggregated SVD is given by
\begin{align}
\Gamma(\mathbf{C}) = \{\Gamma(\mathbf{C}_k)\}_{k=1}^{K}\quad\text{where}\quad \Gamma(\mathbf{C}_k) = \{\mathbf{U}_{r_{\mathbf{C}_k}},\mathbf{\Sigma}_{r_{\mathbf{C}_k}},\mathbf{V}_{r_{\mathbf{C}_k}}\}.
\end{align}
This implies that the distribution of singular values is now cluster-dependent. Clusters with high patch similarity exhibit lower rank, whereas clusters with lower patch similarity exhibit higher rank. A lower rank corresponds to a higher degree of compression, while a higher rank implies reduced compression. That means, for any two randomly selected cluster \(\mathbf{C}_i\) and \(\mathbf{C}_j\) (\(i\neq j\)) with patches in \(\mathbf{C}_i\) exhibiting low-rank and \(\mathbf{C}_j\) exhibiting high rank, we have 
\begin{align}
    r_{\mathbf{C}_i}\ll r_{\mathbf{C}_j}.
\end{align}
For a fixed energy threshold \(\alpha\). It follows that
\begin{align}
    CF_{\mathbf{C}_i}(p, N_{\mathbf{C}_j}, r_{\mathbf{C}_j})\gg CF_{\mathbf{C}_j}(p, N_{\mathbf{C}_j}, r_{\mathbf{C}_j}).
\end{align}
Since cluster \(\mathbf{C}_i\) has a more homogeneous structure than cluster \(\mathbf{C}_j\), its singular values decay more rapidly. This allows for a lower truncation rank \(r_{\mathbf{C}_i}\), leading to a much higher compression factor for \(\mathbf{C}_i\) compared to \(\mathbf{C}_j\). Consequently, the discarded singular values in \(\mathbf{C}_i\) contribute significantly less to the reconstruction error than those in \(\mathbf{C}_j\), resulting in a lower compression loss. That is
\begin{align} \label{cluster-compression} 
\|\mathbf{C}_i - \widehat{\mathbf{C}}_{i}\|_F^2 \ll \|\mathbf{C}_j - \widehat{\mathbf{C}}_{j}\|_F^2,  
\end{align}
where \(\widehat{\mathbf{C}}_i\) denotes the compressed version of the cluster matrix \(\mathbf{C}_i\). Eq.~\ref{cluster-compression} expands to
\begin{align}  
\sum_{k=r_{\mathbf{C}_i}+1}^{\min\{N_{\mathbf{C}_i},p^2\}} \sigma_{\mathbf{C}_i, k}^2 \ll \sum_{k=r_{\mathbf{C}_j}+1}^{\min\{N_{\mathbf{C}_j},p^2\}} \sigma_{\mathbf{C}_j, k}^2.  
\end{align}  

Fig.~\ref{fig:components retained} shows two randomly selected clusters, \(\mathbf{C}_{0}\) and \(\mathbf{C}_{2}\). In this figure, the x-axis represents the number of singular values retained, while the left y-axis displays the cumulative energy captured as these singular values are accumulated, and the right y-axis indicates the corresponding compression factor. Unlike \(r_g\), \(r_{\mathbf{C}_k}\) is determined individually for each cluster based on its specific information content. The cumulative energy curves for two representative clusters show that the number of singular values needed to reach a target energy threshold (e.g., 95\%) varies across clusters. 
\begin{figure*}
    \centering
    \includegraphics[width=.9\textwidth]{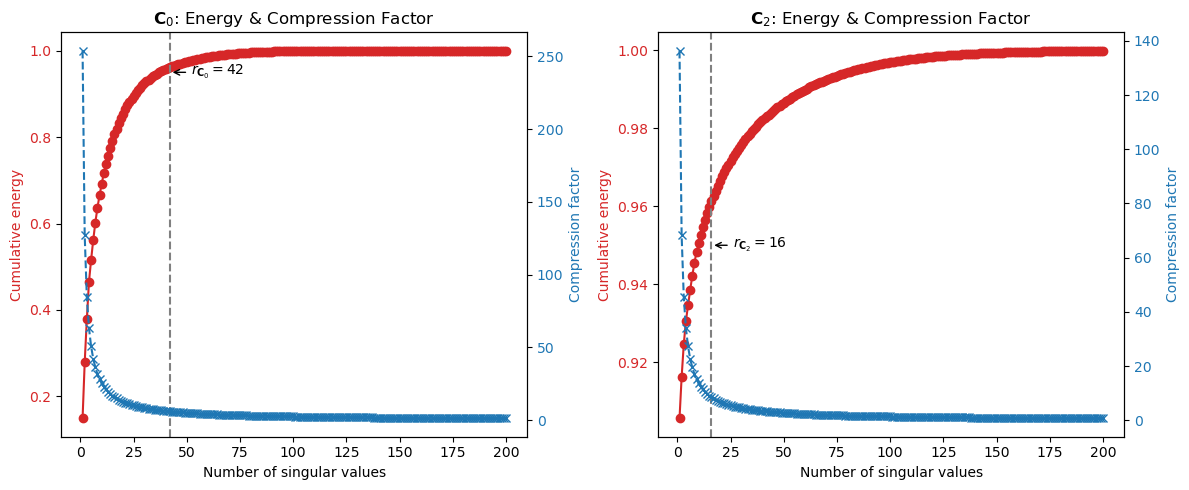}
    \caption{These plots compare how the number of retained singular values (horizontal axis) affects the cumulative energy (red curve, left vertical axis) and the compression factor (blue curve, right vertical axis) for two randomly selected cluster matrices, \(\mathbf{C}_{0}\) (left) and \(\mathbf{C}_{2}\) (right). The dashed vertical lines indicate the minimal number of singular values, \(r_{\mathbf{C}_k}\), needed for 95\% energy retention.}
    \label{fig:components retained}
\end{figure*}

Fig.~\ref{fig:intra-cluster variability} shows the intra-cluster variability, measured using the Frobenius norm, across different number of clusters \(K\). The x-axis represents the number of clusters, while the y-axis quantifies the intra-cluster variability. The results demonstrates that when \( K \) is small, patches with different features are grouped in the same cluster \( \mathcal{C}_k \) resulting in cluster matrices \( \mathbf{C}_k \in \mathbb{R}^{N_{\mathbf{C}_k} \times p^2} \) that exhibit high variability. The variability within \( \mathbf{C}_k \) limits the effectiveness of the low-rank approximation as the SVD of \( \mathbf{C}_k \) cannot efficiently isolate meaningful low-rank structures. In particular, \(\|\mathbf{C}_k - \widehat{\mathbf{C}}_{k}\|_F^2\) remains large for small \( r_{\mathbf{C}_k} \), which leads to poor reconstruction fidelity. As a result, local matrix details and structural information are insufficiently captured, leading to artifacts or blurring in the reconstructed matrix. Conversely, as \( K \) increases the patches within each cluster become more homogeneous reducing the variability in \( \mathbf{C}_k \). This makes it easier to maintain fewer singular values \( r_{\mathbf{C}_k} \) while still achieving accurate approximations. This is reflected in the decay of the singular values, where for more homogeneous matrices the tail singular values \( \{ \sigma_{\mathbf{C}_k, i} \}_{i=r_{\mathbf{C}_i}+1}^{\min\{N_{\mathbf{C}_k}, p^2\}} \) are smaller, reducing the reconstruction error. A larger \( K \) leads to better medical image quality, as the local and non-local self-similarity of the patches is effectively utilized.
\begin{figure}
    \centering
    \includegraphics[width=.9\linewidth]{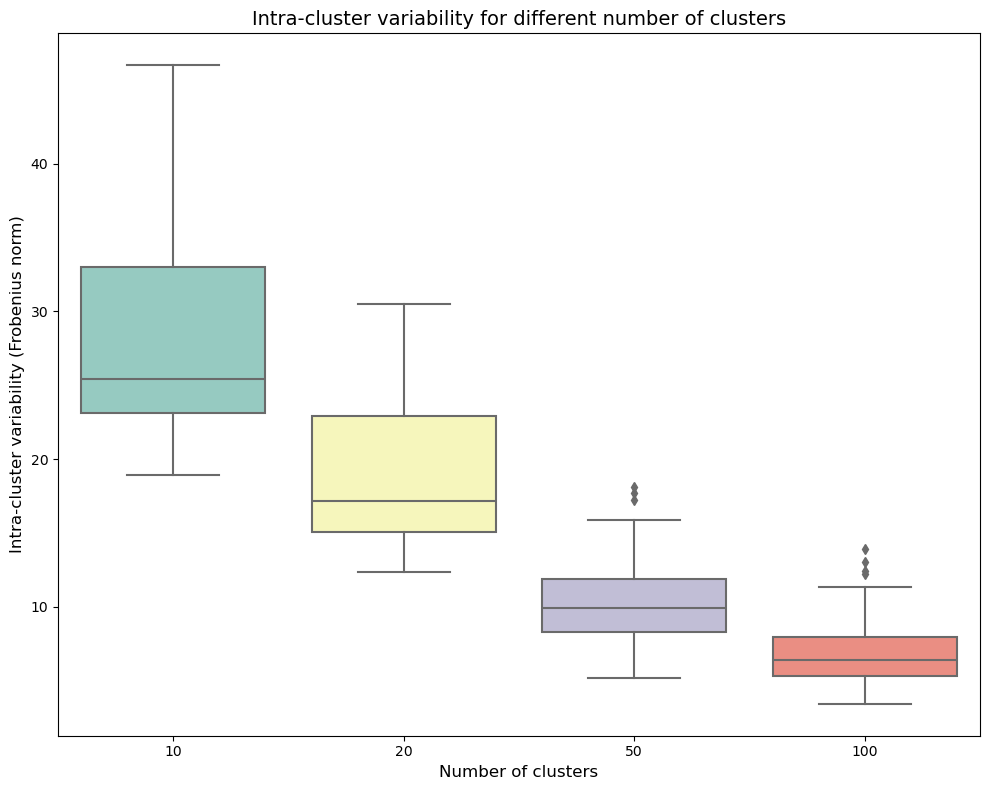}
    \caption{Boxplot showing the intra-cluster variability (Frobenius norm) for different numbers of clusters. As \(K\) increases, the intra-cluster variability decreases, showing a more homogeneous clustering of patches. Smaller \(K\) values lead to higher variability as different patch features are grouped together, while larger \(K\) leads to lower variability as patches with similar features are better captured.}
    \label{fig:intra-cluster variability}
\end{figure}

Fig.~\ref{fig: sv decay} demonstrates heatmaps of the normalized singular value decay for clusters corresponding to different patch sizes. Each subfigure shows both the overall decay pattern and a detailed zoomed-in view of the first 16 singular values, with dashed lines showing regions of interest. This result shows that when \( K \) is optimal, the homogeneity of clusters is significantly influenced by the patch size \( p ^2 \). Smaller patch sizes lead to higher cluster homogeneity by minimizing within-cluster variance and enabling faster singular value decay. Larger patches, on the other hand, exhibit reduced homogeneity due to their inherent variability, which affects the efficiency of cluster-based SVD and the overall quality of the compressed matrix \cite{li2021image}.
\begin{figure*}
    \centering
    \includegraphics[width=.9\textwidth]{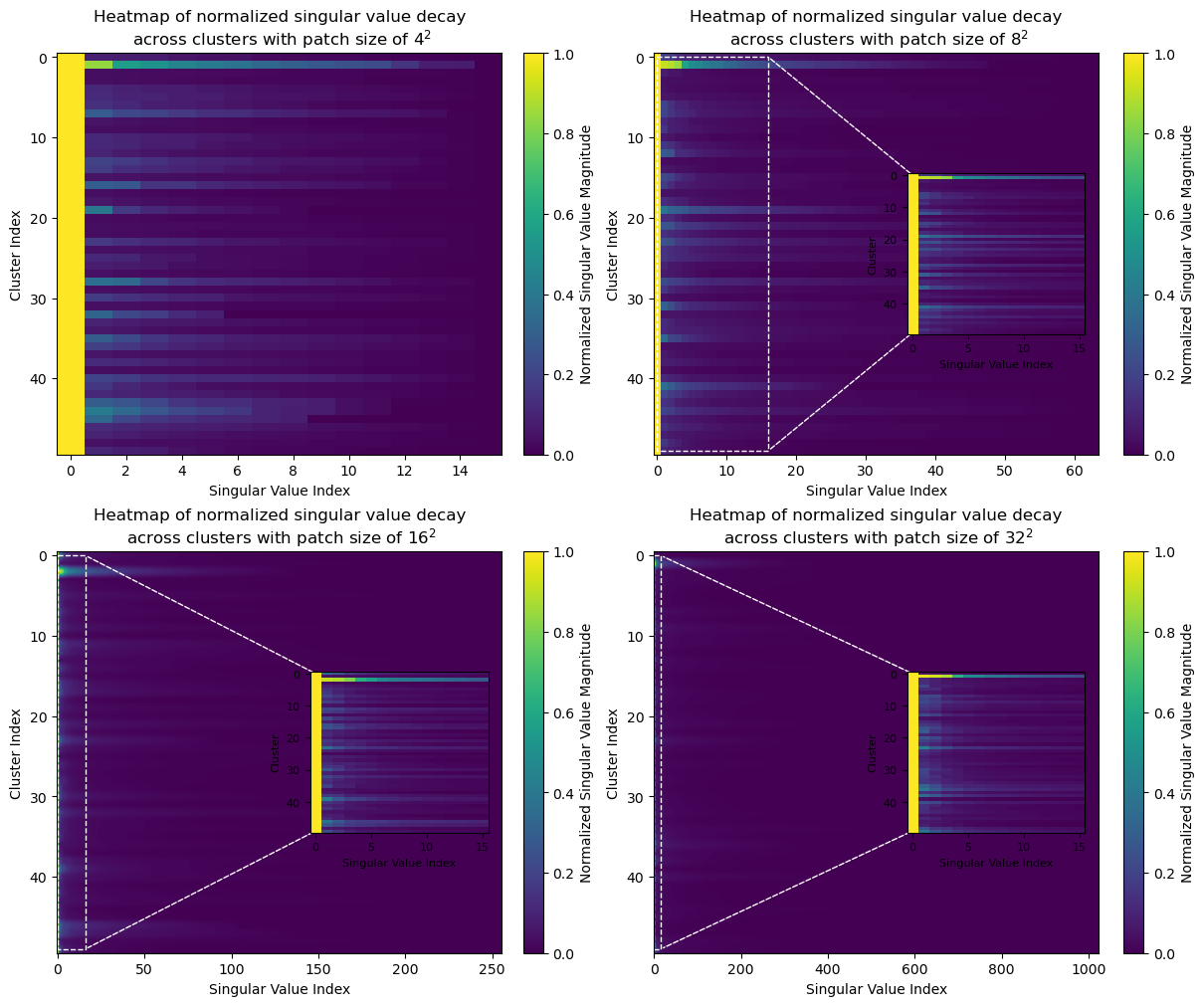}
    \caption{Heatmaps of normalized singular value decay across clusters for various patch sizes. Each subfigure corresponds to a different patch size: \(4^2\), \(8^2\), \(16^2\), and \(32^2\). The main heatmaps illustrate the singular value distribution for all clusters, with zoomed-in views showing the first 16 singular values. Dashed lines connect the highlighted region in the main heatmap to its corresponding zoomed-in view.}
    \label{fig: sv decay}
\end{figure*}
\subsection{The overall compression factor}{\label{sec:compression}}
The compression factor of the medical image \( \mathbf{C}_k \) for cluster \( \mathcal{C}_k \) is determined by the ratio between the number of elements in the original matrix and the number of elements retained after applying truncated SVD. The medical image \( \mathbf{C}_k \in \mathbb{R}^{N_{\mathbf{C}_k} \times p^2} \) contains \( N_{\mathbf{C}_k} \times p^2 \) elements before compression. After applying SVD and retaining only the top \( r_{\mathbf{C}_k} \) singular values, the number of elements in the compressed representation is reduced to \( r_{\mathbf{C}_k} (N_{\mathbf{C}_k} + p^2 + 1) \), which includes \( r_{\mathbf{C}_k} \) singular values, \( r_{\mathbf{C}_k} \times N_{\mathbf{C}_k} \) elements from the left singular matrix \( \mathbf{U}_{r_{\mathbf{C}_k}} \), and \( r_{\mathbf{C}_k} \times p^2 \) elements from the right singular matrix \( \mathbf{V}_{r_{\mathbf{C}_k}}\). Therefore if the original matrix is split into non-overlapping patches then the compression factor for cluster \( \mathcal{C}_k \) is given by
\begin{align}
CF_{\mathbf{C}_k}(p, N_{\mathbf{C}_k},r_{\mathbf{C}_k}) = \frac{ N_{\mathbf{C}_k}  p^2}{r_{\mathbf{C}_k} (N_{\mathbf{C}_k} + p^2 + 1)},    
\end{align}
 and hence the overall compression factor, \(CF_{\text{overall}}(p, K, \alpha)\) is computed as
\begin{align}
    CF_{\text{overall}}(p, K, \alpha)&=\frac{ N p^2}{\sum_{k=1}^{K}r_{\mathbf{C}_k} (N_{\mathbf{C}_k} + p^2 + 1)}\label{overallCF}.
\end{align}
When a matrix is split into overlapping patches, the total number of patches \(N_{\mathbf{C}_k}\) for cluster \( \mathcal{C}_k \) increases due to redundancy, as overlapping regions lead to multiple patches sharing the same data values. To account for the redundancy introduced by overlapping patches, the overlap proportion, \( \beta_{\mathbf{C}_k} \), is calculated as the fraction of overlapping data relative to the total data across all patches in cluster \(\mathcal{C}_k\), that is 
\begin{align}{\label{beta_k}}
    \beta_{\mathbf{C}_k}=\frac{\sum_{i,j}\mathbb{I}(\mathbf{O}_{\mathbf{C}_k}(i,j)>1)}{p^2N_{\mathbf{C}_k}},
\end{align}
where \(\mathbf{O}_{\mathbf{C}_k}\) denotes the overlap map for cluster \(\mathcal{C}_k\), which indicates the number of patches from cluster \(\mathcal{C}_k\) that overlap at \((i,j)\) and \(\mathbb{I}(\cdot)\) denotes an indicator function that equals 1 if \((i,j)\) is covered by more than one patch from cluster \(\mathcal{C}_k\), and 0 otherwise. Using Eq.~\ref{beta_k} the effective number of patches \( N^{\text{eff}}_{\mathbf{C}_k} \) for cluster \( \mathcal{C}_k \), which is defined as an adjustment of the total number of patches, \( N_{\mathbf{C}_k} \), by considering the proportion of overlapping data is computed as 
\begin{align}{\label{unique patches}}
\displaystyle N^{\text{eff}}_{\mathbf{C}_k} \approx \lceil N_{\mathbf{C}_k}  (1 - \beta_{\mathbf{C}_k}) \rceil,
\end{align}
where \( (1 - \beta_{\mathbf{C}_k}) \) represents the proportion of unique, non-overlapping data. Eq.~\ref{unique patches} assumes that the effective patch count decreases linearly with overlap. Algorithm~\ref{alg:overlap_eff_patches} describes how to calculate the overlap proportion and adjust the number of patches to ensure that the redundancy introduced by overlapping patches is properly accounted for when calculating the overall compression factor. 
\begin{algorithm}[!htbp]
\caption{Calculation of \(\beta_{\mathbf{C}_k}\) and \(N^{\text{eff}}\)}
\label{alg:overlap_eff_patches}
\begin{algorithmic}[1]
\State \textbf{Input:} \(\{\mathcal{P}_{\mathbf{C}_k}\}_{k=1}^K\), \(p^2\), \(\text{matrix\_shape}\), \(K\)
\State \textbf{Output:} \(\{\beta_{\mathbf{C}_k}\}_{k=1}^K\), \(N^{\text{eff}}\)
\State \(\beta_{\mathbf{C}_k} \gets 0 \quad  \forall k \in \{1, \dots, K\}\)
\State \(N^{\text{eff}} \gets 0\)
\For{$k \gets 1$ to $K$}
    \State \(\mathbf{O}_{\mathbf{C}_k} \gets \text{zeros}(\text{matrix\_shape})\) \Comment{\(\mathbf{O}_{\mathbf{C}_k}\) is the overlap map for cluster \(\mathcal{C}_k\)}
    \For{$(y, x) \in \mathcal{P}_{\mathbf{C}_k}$} \Comment{\(\mathcal{P}_{\mathbf{C}_k}\) is the set of patch positions for cluster \(\mathcal{C}_k\)}
        \State Increment \(\mathbf{O}_{\mathbf{C}_k}[y:y+p, x:x+p]\) by \(1\) \Comment{Mark overlapping pixels}
    \EndFor
    \State \(\text{overlapping\_pixels}_k \gets \sum (\mathbf{O}_{\mathbf{C}_k} > 1)\) \Comment{Count overlapping pixels}
    \State \(\text{total\_pixels}_k \gets p^2\vert\mathcal{P}_{\mathbf{C}_k} \vert\) \Comment{Total pixels in cluster \(\mathcal{C}_k\) (note $\vert\mathcal{P}_{\mathbf{C}_k}\vert=N_{\mathbf{C}_k}$)}
    \State \(\beta_{\mathbf{C}_k} \gets \frac{\text{overlapping\_pixels}_k}{\text{total\_pixels}_k}\) \Comment{Overlap proportion (Eq.~\ref{beta_k})}
    \State \(N^{\text{eff}}_{\mathbf{C}_k} \gets \vert \mathcal{P}_{\mathbf{C}_k}\vert (1 - \beta_{\mathbf{C}_k})\) \Comment{Effective patches for cluster \(\mathcal{C}_k\) (Eq.~\ref{unique patches})}
    \State \(N^{\text{eff}} \gets N^{\text{eff}} + N^{\text{eff}}_{\mathbf{C}_k}\) \Comment{Summing effective patches (Eq.~\ref{total effective})}
\EndFor
\State \Return \(\{\beta_{\mathbf{C}_k}\}_{k=1}^K\), \(N^{\text{eff}}\)
\end{algorithmic}
\end{algorithm}
This adjustment ensures that \( N^{\text{eff}}_{\mathbf{C}_k} \) more accurately reflects the true amount of information available in cluster \(\mathcal{C}_k\). The effective compression factor for cluster \( \mathcal{C}_k \), \(CF_{\mathbf{C}_k}^{\text{eff}}(p, \alpha, K) \), is approximated as
\begin{align}
CF_{\mathbf{C}_k}^{\text{eff}}(r_{\mathbf{C}_k}, N_{\mathbf{C}_k}^{\text{eff}},p) \approx \frac{N^{\text{eff}}_{\mathbf{C}_k}  p^2}{r_{\mathbf{C}_k}  (N_{\mathbf{C}_k} + p^2+1)}.
\end{align}
It follows that the effective overall compression factor, \( CF_{\text{overall}}^{\text{eff}}(p, \alpha, K) \), is approximated by
\begin{align}{\label{overall eff}}
CF_{\text{overall}}^{\text{eff}}(p, \alpha, K) \approx \frac{N^{\text{eff}}  p^2}{\sum_{k=1}^{K}r_{\mathbf{C}_k}  (N_{\mathbf{C}_k} + p^2+1)}
\end{align}
where \( N^{\text{eff}} \) denotes the effective number of patches, accounting for the unique, non-overlapping data across all clusters by adjusting for the overlap proportion and is defined as
\begin{align}{\label{total effective}}
N^{\text{eff}} &= \sum_{k=1}^{K} N_{k}^{\text{eff}}\\
&= \sum_{k=1}^{K}\lceil N_{\mathbf{C}_k}  \left( 1 - \beta_{\mathbf{C}_k} \right)\rceil.
\end{align}
The global and cluster-based SVD achieve compression differently, their compression factors are rarely equal, therefore a compression threshold \(CF_{\text{threshold}}\), is calculated to keep them approximately equal.
\newpage

\begin{proposition}In general, there does not always exist 
\begin{align*}
r_g \in \mathbb{N}
\end{align*}
such that
\begin{align*}
CF_{\text{global}}(r_g) \approx CF_{\text{overall}}^{\text{eff}}(p, \alpha, K).
\end{align*}
\end{proposition}
\begin{proof}[\textbf{Proof}]\mbox{}\\ 
Define the function
\begin{align*}
f(r_g) = \frac{mn}{r_g(m+n+1)},
\end{align*}
and let
\begin{align*}
C = CF_{\text{overall}}^{\text{eff}}(p, \alpha, K) = \frac{N^{\text{eff}}\, p^2}{\sum_{k} r_{\mathbf{C}_k} (N_{\mathbf{C}_k}+p^2+1)}.
\end{align*}
Now, define the discrete set
\begin{align*}
\mathcal{S} = \{ f(r_g) : r_g \in \mathbb{N},\ 1 \leq r_g \leq \min(m,n) \}.
\end{align*}
Since $\mathcal{S}$ is a finite discrete set in $\mathbb{R}$, there exists a minimum separation
\begin{align*}
\delta = \min \{ \vert x-y \vert : x,y \in \mathcal{S},\ x\neq y \} > 0.
\end{align*}
We analyze two cases:\\
\\
\textbf{Case 1:} $C \in \mathcal{S}$.\\[1mm]
In this case, there exists $r_g^*$ such that
\begin{align*}
f(r_g^*) = C.
\end{align*}
Thus, we can select $r_g = r_g^*$, ensuring that
\begin{align*}
CF_{\text{global}}(r_g) = CF_{\text{overall}}^{\text{eff}}(p, \alpha, K).
\end{align*}
However, because \( C \) is derived from independently determined local quantities, the precise equality can only occur under very special and highly regular conditions. That is, the local patches across the entire medical image would need to be perfectly uniform, meaning that their characteristics such as texture, intensity, and structure are identical in every region. Second, the local compression parameters, including the number of patches \( N_{\mathbf{C}_k} \), the effective patch counts \( N^{\text{eff}}_{\mathbf{C}_k} \), and the singular value truncation ranks \( r_{\mathbf{C}_k} \) for each cluster, must be exactly the same or follow a perfectly regular pattern. Third, the energy thresholds used in the local SVD must align exactly with the global energy distribution so that the overall energy preservation in the patch-based method yields an effective compression factor \( CF_{\text{overall}}^{\text{eff}}(p, \alpha, K) \) that exactly matches one of the discrete values produced by \( f(r_g) \). Finally, the redundancy introduced by overlapping patches must be completely controlled so that the effective patch count \( N^{\text{eff}} \) matches the theoretical prediction from the global formulation.\\
\\
\textbf{Case 2:} $C \notin \mathcal{S}$.\\[1mm]
Since $\mathcal{S}$ is discrete, define
\begin{align*}
\delta_0 = \min_{x\in \mathcal{S}} \vert x-C \vert > 0.
\end{align*}
For any tolerance $\epsilon$, two possibilities arise:
\begin{enumerate}
    \item[(i)] If $\delta_0 > \epsilon$:\\[1mm]
    The tolerance is stricter than the best possible error achievable from $\mathcal{S}$. In this case, no $r_g$ satisfies
    \begin{align*}
    \vert f(r_g)-C\vert < \epsilon.
    \end{align*}
    That is, we cannot find an $r_g$ that makes the compression factors equal within the given tolerance.
    
    \item[(ii)] If $\delta_0 \leq \epsilon$:\\[1mm]
    The tolerance is loose enough to accept the best achievable approximation from $\mathcal{S}$. Then there exists some $r_g$ such that
    \begin{align*}
    \vert f(r_g)-C\vert \leq \epsilon.
    \end{align*}
    In this case, we can find an $r_g$ that brings the compression factors sufficiently close.
\end{enumerate}
Since the value of $C$ depends on local patch-based characteristics and is arbitrary with respect to $\mathcal{S}$, we conclude that there does not always exist an $r_g \in \mathbb{N}$ such that
\begin{align*}
CF_{\text{global}}(r_g) \approx CF_{\text{overall}}^{\text{eff}}(p, \alpha, K).
\end{align*}
\end{proof}
The \(CF_{\text{threshold}}\) is defined as the relative percentage difference between \(CF_{\text{overall}}^{\text{eff}}(p, \alpha, K)\) and \(CF_{\text{global}}(r_g)\). That is,
\begin{align}
    CF_{\text{threshold}} = \frac{\vert CF_{\text{overall}}^{\text{eff}}(p, \alpha, K) -CF_{\text{global}}(r_g)\vert}{CF_{\text{overall}}^{\text{eff}}(p, \alpha, K)},
\end{align}
where \(\vert \cdot \vert\) denotes the absolute value. \(CF_{\text{threshold}}\) specifies the maximum allowable value of this relative difference so that the CFs are considered approximately equal. For example, if \(CF_{\text{threshold}}\) is set to 0.15, this means that the CFs for global and cluster-based SVD must not differ by more than 15\%.\\
\\
\begin{proposition} Assume that for every cluster \(\mathcal{C}_k\), the number of patches \(N_{\mathbf{C}_k}(p)\) and the effective patch count \(N^{\mathrm{eff}}(p)\) are non-increasing in \(p\), and that
\begin{align*}
\phi_{\mathbf{C}_k}(p)=r_{\mathbf{C}_k}(p)\left(1+\frac{N_{\mathbf{C}_k}(p)+1}{p^2}\right)
\end{align*}
is non-increasing in \(p\). Then the effective overall compression factor \(CF^{\mathrm{eff}}_{\mathrm{overall}}(p,\alpha,K)\) is an increasing function of \(p\) for fixed \(\alpha\) and \(K\).
\end{proposition}
\begin{proof}[\textbf{Proof}]\mbox{}\\ Define
\begin{align*}
CF^{\mathrm{eff}}_{\mathrm{overall}}(p, \alpha, K)=\frac{N^{\mathrm{eff}}(p)\,p^2}{\displaystyle\sum_{k=1}^{K} r_{\mathbf{C}_k}(p)\bigl(N_{\mathbf{C}_k}(p)+p^2+1\bigr)}.
\end{align*}
Since
\begin{align*}
N_{\mathbf{C}_k}(p)+p^2+1 = p^2\left(1+\frac{N_{\mathbf{C}_k}(p)+1}{p^2}\right),
\end{align*}
we can rewrite the overall effective compression factor as follows:
\begin{align*}
CF^{\mathrm{eff}}_{\mathrm{overall}}(p, \alpha, K)=\frac{N^{\mathrm{eff}}(p)p^2}{\displaystyle\sum_{k=1}^{K} r_{\mathbf{C}_k}(p)\,\psi_{\mathbf{C}_k}(p)},
\end{align*}
where
\begin{align*}
\psi_{\mathbf{C}_k}(p)=1+\frac{N_{\mathbf{C}_k}(p)+1}{p^2}.
\end{align*}
For each \(k\) the function \(N_{\mathbf{C}_k}(p)\) is non-increasing in \(p\) while \(p^2\) is strictly increasing; hence, each \(\psi_{\mathbf{C}_k}(p)\) is strictly decreasing. Although the function \(r_{\mathbf{C}_k}(p)\) (which is determined by Eq.~\ref{threshold}) need not be monotonic in \(p\), we assume that for every \(k\), the combined function
\begin{align*}
\phi_{\mathbf{C}_k}(p)=r_{\mathbf{C}_k}(p)\,\psi_{\mathbf{C}_k}(p)
\end{align*}
is non-increasing in \(p\). That is, for any \(1<p_1<p_2\) we have
\begin{align*}
\phi_{\mathbf{C}_k}(p_1)=r_{\mathbf{C}_k}(p_1)\,\psi_{\mathbf{C}_k}(p_1)\ge r_{\mathbf{C}_k}(p_2)\,\psi_{\mathbf{C}_k}(p_2)=\phi_{\mathbf{C}_k}(p_2).
\end{align*}
The strict decrease of \(\psi_{\mathbf{C}_k}(p)\) due to the quadratic increase of \(p^2\) and the decrease of \(N_{\mathbf{C}_k}(p)\) ensures that even if \(r_{\mathbf{C}_k}(p)\) does not decrease, the product \(r_{\mathbf{C}_k}(p)\,\psi_{\mathbf{C}_k}(p)\) will tend to decrease as \(p\) increases. Summing over \(k\) yields
\begin{align*}
\sum_{k=1}^{K} r_{\mathbf{C}_k}(p_1)\,\psi_{\mathbf{C}_k}(p_1)\ge \sum_{k=1}^{K} r_{\mathbf{C}_k}(p_2)\,\psi_{\mathbf{C}_k}(p_2).
\end{align*}
Although \(N^{\mathrm{eff}}(p)\) is non-increasing in \(p\), we assume that the decrease in the denominator which is driven by the quadratic decay in \(\psi_{\mathbf{C}_k}(p)\) dominates the decrease in \(N^{\mathrm{eff}}(p)\). That is, for any \(1<p_1<p_2\),
\begin{align*}
\frac{N^{\mathrm{eff}}(p_1)p^2_1}{\displaystyle\sum_{k=1}^{K} r_{\mathbf{C}_k}(p_1)\,\psi_{\mathbf{C}_k}(p_1)} < \frac{N^{\mathrm{eff}}(p_2)p^2_2}{\displaystyle\sum_{k=1}^{K} r_{\mathbf{C}_k}(p_2)\,\psi_{\mathbf{C}_k}(p_2)}.
\end{align*}
Therefore,
\begin{align*}
CF^{\mathrm{eff}}_{\mathrm{overall}}(p_1, \alpha, K) < CF^{\mathrm{eff}}_{\mathrm{overall}}(p_2, \alpha, K).
\end{align*}
\end{proof}

 Algorithm~\ref{alg:cf_eff} calculates the effective overall compression factor, \(CF_{\text{overall}}^{\text{eff}}(p, \alpha, K)\), by considering overlapping patches and determining the number of singular values to retain for each cluster based on the energy threshold, \(\alpha\).
\begin{algorithm}[!htbp]
\caption{Calculation of $CF^{\text{eff}}_{\text{overall}}$}
\label{alg:cf_eff}
\begin{algorithmic}[1]
\State \textbf{Input:} $\{N_{\mathbf{C}_k}\}_{k=1}^K$, $p^2$, $\{\beta_{\mathbf{C}_k}\}_{k=1}^K$, $\{\sigma_i^2\}_{i=1}^{\min\{N_{\mathbf{C}_k},p^2\}}$, $\alpha$, $K$
\State \textbf{Output:} $CF^{\text{eff}}_{\text{overall}}$
\State $N_{\text{eff}} \gets 0$, $D \gets 0$ \Comment{$D$ represents the denominator in (Eq.~\ref{overall eff})}
\For{$k \gets 1$ to $K$} 
    \State $N^{\text{eff}}_{\mathbf{C}_k} \gets N_{\mathbf{C}_k}  (1 - \beta_{\mathbf{C}_k})$ \Comment{Effective patches for cluster \(\mathcal{C}_k\) (Eq.~\ref{unique patches})}
    \State $N^{\text{eff}} \gets N^{\text{eff}} + N^{\text{eff}}_{\mathbf{C}_k}$ \Comment{Summing effective patches (Eq.~\ref{total effective})}
    \State $E_{\mathbf{C}_k} \gets \sum_{i=1}^{\min\{N_{\mathbf{C}_k},p^2\}} \sigma_i^2$ \Comment{$E_{\mathbf{C}_k}$ is total energy for cluster \(\mathcal{C}_k\)}
    \State $r_{\mathbf{C}_k} \gets 0$, $P_{\mathbf{C}_k} \gets 0$ \Comment{$P_{\mathbf{C}_k}$ is partial energy for cluster \(\mathcal{C}_k\)}
    \While{$P_{\mathbf{C}_k} < \alpha E_{\mathbf{C}_k}$} \Comment{Energy threshold (Eq.~\ref{retained components})}
        \State $r_{\mathbf{C}_k} \gets r_{\mathbf{C}_k} + 1$
        \State $P_{\mathbf{C}_k} \gets \sum_{i=1}^{r_{\mathbf{C}_k}} \sigma_i^2$
    \EndWhile
    \State $C_{\mathbf{C}_k} \gets r_{\mathbf{C}_k}  (N_{\mathbf{C}_k} + p^2 + 1)$ \Comment{$C_{\mathbf{C}_k}$ is cluster contribution for $k$}
    \State $D \gets D + C_{\mathbf{C}_k}$ \Comment{Summing cluster contributions for denominator}
\EndFor
\State $CF^{\text{eff}}_{\text{overall}} \gets \displaystyle\frac{p^2 N^{\text{eff}}}{D}$ \Comment{Effective overall $CF$ (Eq.~\ref{overall eff})}
\State \Return $CF^{\text{eff}}_{\text{overall}}$
\end{algorithmic}
\end{algorithm}
\subsection{Data matrix  reconstruction}
The compressed medical image, \(\widehat{\mathbf{A}}\), is reconstructed by combining the compressed patches from each cluster back into their original positions. In the compressed representation, each row of \(\widehat{\mathbf{C}}_k\) corresponds to a vectorized patch extracted from the original matrix \(\mathbf{A}\) at a specific position, determined by the top-left corner coordinates \((i_k,j_k)\) of the patch. For the \(l^{\text{th}}\) patch in cluster \(\mathcal{C}_k\), let \(\Omega_{k,l} \subset [1,m] \times [1,n]\) be its spatial support, that is, the set of pixel coordinates covered by the patch. An indicator function is defined as
\begin{align}
\mathbb{I}\{(i,j) \in \Omega_{k,l}\} = \begin{cases} 
1, & \text{if } (i,j) \in \Omega_{k,l}\\
0, & \text{otherwise}.
\end{cases}
\end{align}
Since the medical image is split into overlapping patches, a pixel at location \((i,j)\) may be covered by several patches. Let \(w_{ij}\) denote the number of patches covering pixel \((i,j)\). The final reconstructed image \(\widehat{\mathbf{A}} \in \mathbb{R}^{m \times n}\) is computed as
\begin{align}
\widehat{\mathbf{A}}_{ij} = \frac{1}{w_{ij}} \sum_{k=1}^{K} \sum_{l=1}^{N_{\mathbf{C}_k}} \mathbb{I}\{(i,j) \in \Omega_{k,l}\}\, \Bigl[\operatorname{vec}^{-1}\Bigl([\,\widehat{\mathbf{C}}_k\,]_l\Bigr)\Bigr]_{ij},
\end{align}
where \([\,\widehat{\mathbf{C}}_k\,]_l\) denotes the \(l^{\text{th}}\) row of \(\widehat{\mathbf{C}}_k\), the compressed representation of the \(l^{\text{th}}\) patch in cluster \(\mathcal{C}_k\), and \(\Bigl[\operatorname{vec}^{-1}([\,\widehat{\mathbf{C}}_k\,]_l)\Bigr]_{ij}\) is the pixel value at location \((i,j)\) in the reshaped \(p\times p\) patch.
\subsection{Computation complexity}\label{sec:computational cost}
The cost of k-means clustering is affected by the number of iterations \(T\), the total number of patches \(N\), the dimensionality of the patches \(p^2\), and the number of clusters \(K\). The complexity for each iteration of k-means clustering is approximately \(\mathcal{O}(KNp^2)\) \cite{alsabti1997efficient, huang1998extensions}. For \(T\) iterations, the total cost for k-means clustering is
\begin{align}
\text{cost}_{\text{kmeans}}\sim\mathcal{O}(TKNp^2).    
\end{align}
After clustering, each cluster \(\mathcal{C}_k\) has a matrix \(\mathbf{C}_k \in \mathbb{R}^{N_{\mathbf{C}_k} \times p^2}\), where \(N_{\mathbf{C}_k}\) is the number of patches in cluster \(\mathcal{C}_k\). The computational complexity of SVD for each matrix is 
\begin{align}
\text{cost}_{\text{full SVD}}\sim\mathcal{O}(N_{\mathbf{C}_k} p^2  \min\{N_{\mathbf{C}_k}, p^2\}).
\end{align}
If \(r_{\mathbf{C}_k}\) singular values are retained for each cluster, the complexity of compressing with truncated SVD in each cluster becomes 
\begin{align}
\text{cost}_{\text{t-SVD}}\sim\mathcal{O}(N_{\mathbf{C}_k} p^2 r_{\mathbf{C}_k}),    
\end{align}
where \(r_{\mathbf{C}_k} \ll \min\{N_{\mathbf{C}_k}, p^2\}\). Therefore, the overall complexity for all \(K\) clusters for SVD becomes
\begin{align}
\text{cost}_{\text{total SVD}}\sim\mathcal{O}\left(p^2\sum_{k=1}^{K} N_{\mathbf{C}_k} r_{\mathbf{C}_k} \right).
\end{align}
Additional operations, such as splitting the matrix into patches and reconstructing the matrix after compression, involve smaller computational costs. These can be represented collectively as \(\mathcal{O}(\epsilon)\). It then follows that the total computational cost is given by 
\begin{align}{\label{total cost}}
{\text{total cost}} \sim \mathcal{O}\left(TKNp^2 + p^2\sum_{k=1}^{K} N_{\mathbf{C}_k} r_{\mathbf{C}_k} +\epsilon\right).
\end{align}
The computational cost in Eq.~\ref{total cost} is lower for non-overlapping patches than for overlapping patches because non-overlapping patches eliminate redundancy caused by overlapping regions. Overlapping patches significantly increase the total number of patches \(N\), as each pixel can belong to multiple patches. This increases the complexity of clustering, which scales as \(\mathcal{O}(TKNp^2)\), and of SVD, which scales as \(\mathcal{O}\left(p^2 \sum_{k=1}^{K} N_{\mathbf{C}_k} r_{\mathbf{C}_k}\right)\). In addition, overlapping patches require a weighted averaging process during reconstruction to resolve contributions from multiple overlapping patches, further adding to the computational overhead. In contrast, non-overlapping patches simplify the process, as each pixel belongs to exactly one patch, reducing \(N\) and avoiding the need for overlap-related operations. Consequently, the terms \(TKNp^2\) and \(\sum_{k=1}^{K} N_{\mathbf{C}_k} r_{\mathbf{C}_k}\) in Eq.~\ref{total cost} are smaller for non-overlapping patches, resulting in a lower overall computational cost. 

If parallel computing is applied to both k-means clustering and SVD, the overall computational complexity in Eq.~\ref{total cost} can be significantly reduced. For k-means clustering, assuming it is parallelized efficiently, the complexity can be reduced depending on the number of processors \( P \). In an ideal scenario, where computations are evenly distributed across processors, the complexity can be reduced to 
\begin{align}
 \text{cost}_{\text{kmeans parallel}}\sim   \mathcal{O}\left(\frac{TKNp^2}{P}\right),
\end{align} leading to a substantial speedup. For the SVD step, parallelization can be implemented by processing each cluster independently on separate processors. In this case, the overall complexity is determined by the largest computational burden among the clusters, reducing the complexity to approximately 
\begin{align}
 \text{cost}_{\text{t-SVD parallel}}\sim\mathcal{O}\left(p^2 \max_k \{N_{\mathbf{C}_k} r_{\mathbf{C}_k}\} \right).
\end{align}
Therefore, the total computational cost with parallel computing can be approximated as
\begin{align}
 \text{total cost}_{\text{parallel}}\sim\mathcal{O}\left(\frac{TKNp^2}{P} + p^2 \max_k \{N_{\mathbf{C}_k} r_{\mathbf{C}_k}\} + \epsilon \right).
\end{align}
\section{Evaluation Metrics}{\label{sec:evaluation}}
This work uses several metrics to assess the quality of the reconstructed medical image. The mean square error (MSE) quantifies the average squared difference between corresponding pixels of the original matrix \(\mathbf{A}\) and its reconstruction \(\widehat{\mathbf{A}}\) over a region of interest (ROI) \(\Omega \subset [1,m] \times [1,n]\). It is computed as
\begin{align}
\text{MSE}(\mathbf{A}_{\Omega}, \widehat{\mathbf{A}}_{\Omega}) = \frac{1}{\vert \Omega \vert} \sum_{(i,j) \in \Omega} (\mathbf{A}_{ij} - \widehat{\mathbf{A}}_{ij})^2.
\end{align}
In addition to MSE, the structural similarity index measure (SSIM) evaluates matrix quality by considering structural information, luminance, and contrast. SSIM is defined by
\begin{align}
\text{SSIM}(\mathbf{A}_\Omega, \widehat{\mathbf{A}}_\Omega) = \frac{(2\mu_{\mathbf{A}_\Omega}\mu_{\widehat{\mathbf{A}}_\Omega} + C_1)(2\sigma_{\mathbf{A}_\Omega\widehat{\mathbf{A}}_\Omega} + C_2)}{(\mu_{\mathbf{A}_\Omega}^2 + \mu_{\widehat{\mathbf{A}}_\Omega}^2 + C_1)(\sigma_{\mathbf{A}_\Omega}^2 + \sigma_{\widehat{\mathbf{A}}_\Omega}^2 + C_2)},
\end{align}
where the means, variances, and covariance capture local matrix statistics and \(C_1\), \(C_2\) are stabilizing constants. The peak signal-to-noise ratio (PSNR) further characterizes reconstruction quality in decibels by comparing the maximum pixel value with the MSE:
\begin{align}
\text{PSNR}(\mathbf{A}_\Omega, \widehat{\mathbf{A}}_\Omega) = 10 \log_{10} \left( \frac{\text{MAX}^2}{\text{MSE}(\mathbf{A}_\Omega, \widehat{\mathbf{A}}_\Omega)} \right).
\end{align}

For segmentation performance, especially in ROI, the intersection over union (IoU) metric is used and is defined as
\begin{align}
\text{IoU}(\Omega, \widehat{\Omega}) = \frac{\vert \Omega \cap \widehat{\Omega}\vert}{\vert \Omega \cup \widehat{\Omega}\vert}.
\end{align} While the relative error is used to compare the performance of cluster-based and global SVD and is defined as
\begin{align}
    \text{Relative error}_{\Omega}=\frac{\Vert\mathbf{A}_{\Omega}-\widehat{\mathbf{A}}_{\Omega}\Vert_F}{\Vert\mathbf{A}_{\Omega}\Vert_F}.
\end{align}

Edge detection and preservation, critical for identifying anatomical structures and abnormalities \cite{wilson11995medical, jiang2010medical, nikolic2016edge}, are evaluated using the Sobel filter \cite{vincent2009descriptive}. This method applies horizontal and vertical convolution kernels to obtain gradient maps \(\mathbf{S}_x\) and \(\mathbf{S}_y\). The edge preservation index (EPI) is then calculated as
\begin{align}
\mathrm{EPI} &= 1 - \frac{\Vert \mathbf{S}_{\mathbf{A}} - \widehat{\mathbf{S}}_{\widehat{\mathbf{A}}} \Vert_1}{\Vert \mathbf{S}_{\mathbf{A}} \Vert_1}\\ 
&= 1 - \frac{\sum_{(i,j) \in \Omega} \vert \mathbf{S}_{\mathbf{A}}(i,j) - \widehat{\mathbf{S}}_{\widehat{\mathbf{A}}}(i,j) \vert}{\sum_{(i,j) \in \Omega} \vert \mathbf{S}_{\mathbf{A}}(i,j) \vert},
\end{align}
where \(\Vert \cdot\Vert_1\) denotes the entry-wise \(\ell_1\) norm and \(\mathbf{S}_{\mathbf{A}}(i,j)\) denotes the gradient magnitude for pixel \((i,j)\) and is defined as
\begin{align}
\mathbf{S}_{\mathbf{A}}(i,j) = \sqrt{(\mathbf{S}_x(i,j))^2 + (\mathbf{S}_y(i,j))^2},
\end{align}
with a similar expression for the reconstructed medical image. An EPI value of 1 indicates perfect preservation of edge information, while values closer to 0 suggest significant degradation.
\section{Experimental Setting}{\label{sec:experiment}}
Medical images are inherently large and contain complex structural details that must be preserved for accurate diagnosis, making efficient compression techniques crucial. Additionally, the high redundancy in medical imaging data makes it well-suited for low-rank approximation methods, which can reduce storage while maintaining critical diagnostic information. Therefore, we conducted experiments on medical images with different modalities to evaluate the effectiveness of global and cluster-based SVD in image compression and feature extraction. All computations were performed using a computer equipped with an 11th Gen Intel Core i7-1165G7 processor (2.80 GHz, 2.70 GHz base clock) and 16 GB RAM (15.6 GB usable). The system runs on a 64-bit operating system (x64 architecture). The global and cluster-based SVD were implemented using Python with libraries such as NumPy, SciPy, and OpenCV for matrix operations and image processing. To optimize computational efficiency, Joblib \cite{joblib} was used for parallel computing, allowing k-means clustering and SVD calculations to be executed across multiple processors. The experimental results discussed in this section are specific to MRI. Outcomes for other imaging modalities such as ultrasound, CT scan, and X-ray are discussed in the appendix. Readers are referred to the appendix for detailed findings across these additional modalities.
\subsection{Dataset Description}
In this experiment we used the Brain Tumor Segmentation (BraTS) 2020 dataset, which is publicly available on \href{https://www.kaggle.com/datasets/awsaf49/brats20-dataset-training-validation}{Kaggle} \cite{menze2014multimodal, bakas2017advancing, bakas2018identifying}. The BraTS 2020 dataset is widely recognized in medical image analysis, particularly for brain tumor segmentation and classification. It consists of high-resolution, multi-modal MRI scans including T1-weighted, T1-contrast enhanced (T1ce), T2-weighted, and FLAIR sequences—acquired using standardized MRI protocols at multiple clinical centers. Each patient underwent a series of scans to capture a range of tissue contrasts, and the raw imaging data were preprocessed through steps such as skull stripping, co-registration, and intensity normalization to reduce inter-scanner variability and align the images to a common anatomical space. Expert radiologists then manually annotated the tumor regions, providing detailed segmentation masks that serve as the ground truth for evaluating image analysis algorithms. This dataset is ideal for evaluating compression and denoising techniques due to its high-resolution images and detailed structural features. We selected BraTS 2020 for its clinical relevance, as brain tumor detection is a critical task in medical imaging, and maintaining diagnostic accuracy during compression is essential. For this experiment, we used a single 2D slice from the middle of the MRI volume, ensuring consistency across different evaluations. The corresponding tumor segmentation mask was included in our analysis to assess the impact of compression separately on tumor and non-tumor region.
\subsection{Performance Evaluation}
We evaluate the performance of global and cluster-based SVD on a randomly selected MRI image by analyzing their ability to preserve image quality at different compression levels. We examine how each method balances compression efficiency with preserving and extracting critical image features, particularly in the tumor region. 

Fig.~\ref{fig:diff_tumor_region} shows two MRI scans, each with a tumor located in a distinct brain region, and compares two compression techniques. For each scan, the original image is shown alongside a binary mask that highlights the tumor, followed by images reconstructed using both global and cluster-based SVD at a compression factor of approximately 118 for Figs.~\ref{fig:diff_tumor_region}(a) to \ref{fig:diff_tumor_region}(h) and a compression factor of about 64 for Figs.~\ref{fig:diff_tumor_region}(i) to \ref{fig:diff_tumor_region}(p). Additionally, residual maps for the tumor and non-tumor regions are provided to illustrate the reconstruction errors. The global SVD-compressed image shows significant block artifacts (Figs.~\ref{fig:diff_tumor_region}(c) and ~\ref{fig:diff_tumor_region}(k)), particularly affecting the tumor region and the entire image structure. This suggests that global SVD has difficulty preserving fine-grained local variations, resulting in significant residual errors inside and outside the tumor region (Figs.~\ref{fig:diff_tumor_region}(e), \ref{fig:diff_tumor_region}(g), \ref{fig:diff_tumor_region}(m), and \ref{fig:diff_tumor_region}(o)). In contrast, cluster-based SVD preserves tumor features and background details (Figs.~\ref{fig:diff_tumor_region}(d) and ~\ref{fig:diff_tumor_region}(l)). This difference can be observed in the residual error maps. Global SVD leads to higher errors in the entire image (Figs.~\ref{fig:diff_tumor_region}(e), \ref{fig:diff_tumor_region}(g), \ref{fig:diff_tumor_region}(m), and \ref{fig:diff_tumor_region}(o)), while cluster-based SVD's residual errors are lower (Figs.~\ref{fig:diff_tumor_region}(f), \ref{fig:diff_tumor_region}(h), \ref{fig:diff_tumor_region}(n), and \ref{fig:diff_tumor_region}(p)), indicating better reconstruction quality when the original image is highly compressed. The results indicate that cluster-based SVD preserves localized structures, particularly within the tumor region, while global SVD exhibits more structured artifacts.
\begin{figure}
    \centering
    \includegraphics[width=0.9\linewidth]{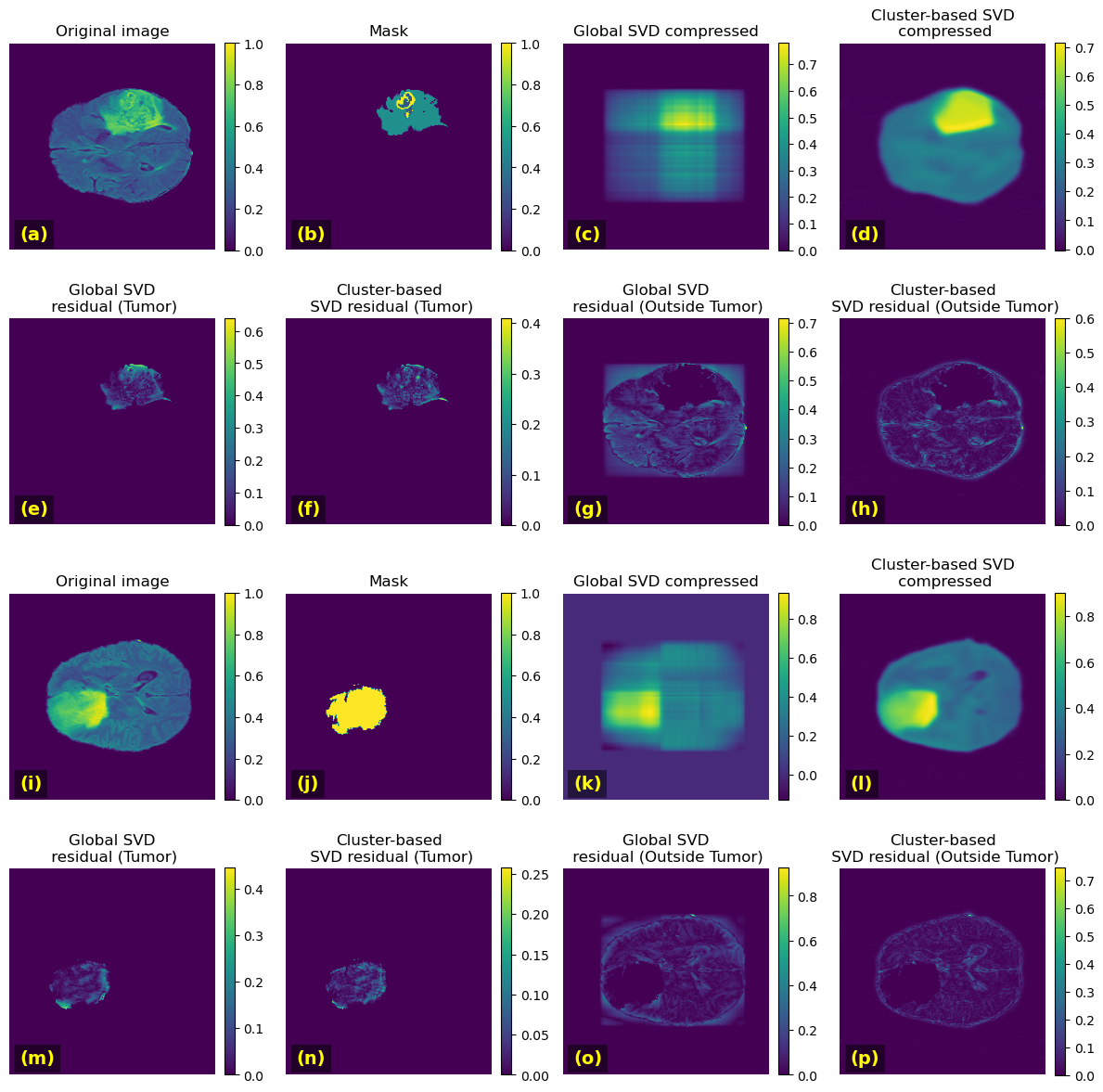}
    \caption{Comparison of global and cluster-based SVD compression for two MRI brain scans with tumors located in different regions. For the first case (a-h), the tumor is situated in a specific brain region. (a) shows the original MRI image, and (b) presents the corresponding binary mask highlighting the tumor area. (c) and (d) display the reconstructed images using global SVD and cluster-based SVD methods, respectively, at a compression factor of approximately 118. The residual maps within the tumor region are shown in (e) and (f) for global and cluster-based approaches, respectively, while (g) and (h) illustrate the residuals outside the tumor area. For the second case (i-p), the tumor appears in a different brain region. (i) presents the original image, and (j) shows the corresponding tumor mask. The compressed reconstructions using global SVD and cluster-based SVD at a compression factor of about 64 are shown in (k) and (l), respectively. (m) and (n) depict the residuals within the tumor area for the global and cluster-based methods, while (o) and (p) show the residuals outside the tumor area.}
    \label{fig:diff_tumor_region}
\end{figure}

The results presented in Figs.~\ref{fig:MRI_tumor} to \ref{fig:edge detection} were obtained using the MRI image shown in Fig.~\ref{fig:diff_tumor_region}(a). To ensure consistency in the evaluation, all performance metrics and comparisons in this section were based on this specific image. The results for MRI image in Fig.~\ref{fig:diff_tumor_region}(i) are included in the appendix.

Fig.~\ref{fig:MRI_tumor} shows a comparative analysis of global SVD and cluster-based SVD in preserving and extracting important image quality features within the tumor region across different compression factors. Figs.~\ref{fig:MRI_tumor}(a), \ref{fig:MRI_tumor}(b), \ref{fig:MRI_tumor}(c) and \ref{fig:MRI_tumor}(d) illustrate the variation of PSNR, SSIM, MSE, and IoU as compression increases, respectively. The x-axis in each plot represents the compression factor, while the y-axis corresponds to the respective performance metric. To ensure a fair comparison, a \(CF_{threshold}\) of 0.15 was used. All the results in Figs.~\ref{fig:MRI_tumor}(a), \ref{fig:MRI_tumor}(b), \ref{fig:MRI_tumor}(c) and \ref{fig:MRI_tumor}(d) show that global SVD cannot compress the tumor region of the image beyond a CF of approximately 120. This results from the fact that global SVD cannot find a valid rank \(r_g\) beyond this threshold to achieve higher compression, reflecting the limitations of its global approach. In contrast, cluster-based SVD can compress the tumor region to much higher CFs while maintaining high performance on all metrics. For PSNR, cluster-based SVD achieves higher scores across all CFs, indicating significantly better reconstruction quality, while global SVD experiences a sharp drop in quality at higher CFs. Similarly, SSIM shows that cluster-based SVD preserves the structural details in the tumor region far more effectively, while global SVD shows significant degradation, especially as CF approaches its upper limit. The MSE results are consistent with these trends, as cluster-based SVD consistently achieves lower error values, highlighting its robustness in maintaining accuracy in the tumor region even under aggressive compression. Finally, the IoU metric confirms that cluster-based SVD provides a more accurate reconstruction of the tumor region, ensuring a closer match to the ground truth. In contrast, the IoU performance of global SVD degrades significantly with increasing CF, reflecting its inability to preserve localized features.
\begin{figure*}
    \centering
    \includegraphics[width=1.0\textwidth]{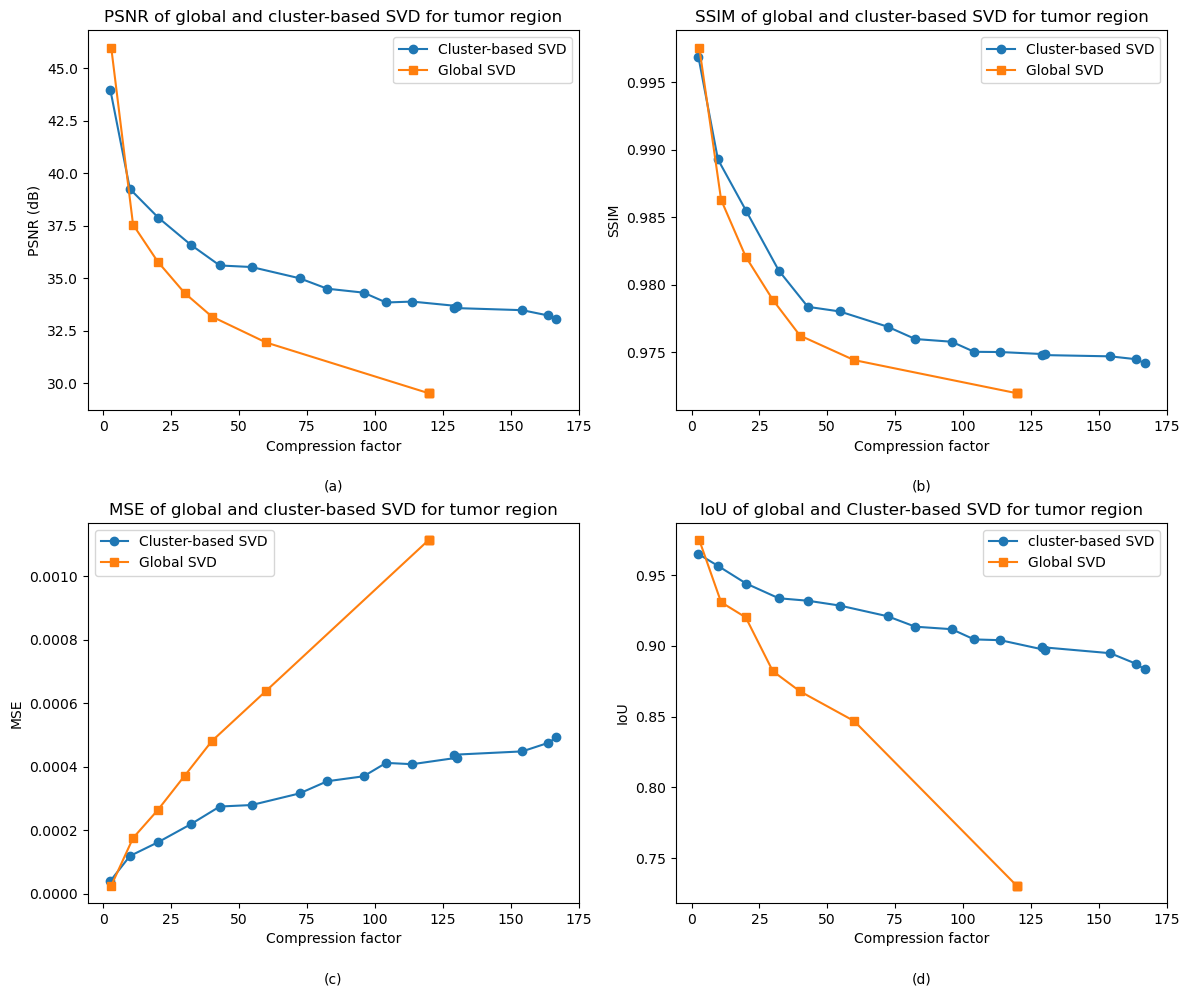}
    \caption{Performance comparison of global SVD and cluster-based SVD for the tumor region across different compression factors. (a) PSNR (dB) vs. Compression factor, (b) SSIM vs. Compression factor, (c) MSE vs. Compression factor, and (d) IoU vs. Compression factor.}
    \label{fig:MRI_tumor}
\end{figure*}

Fig.~\ref{fig:MRI_nontumor} shows similar performance comparisons as Fig.~\ref{fig:MRI_tumor} but focuses on the non-tumor region of the MRI image. This figure exhibits behavior similar to that seen in Fig.~\ref{fig:MRI_tumor}. However, it is noteworthy that the cluster-based SVD technique compresses the reconstructed image more in the non-tumor region compared to tumor region. This increased compression in the non-critical areas is advantageous because it allows for more efficient data reduction without compromising the overall diagnostic quality. By compressing the non-tumor regions more aggressively, the method effectively prioritizes the preservation of image fidelity in the tumor region (i.e., the primary area of interest) while optimizing storage and processing efficiency elsewhere.
\begin{figure*}
    \centering
    \includegraphics[width=1.0\textwidth]{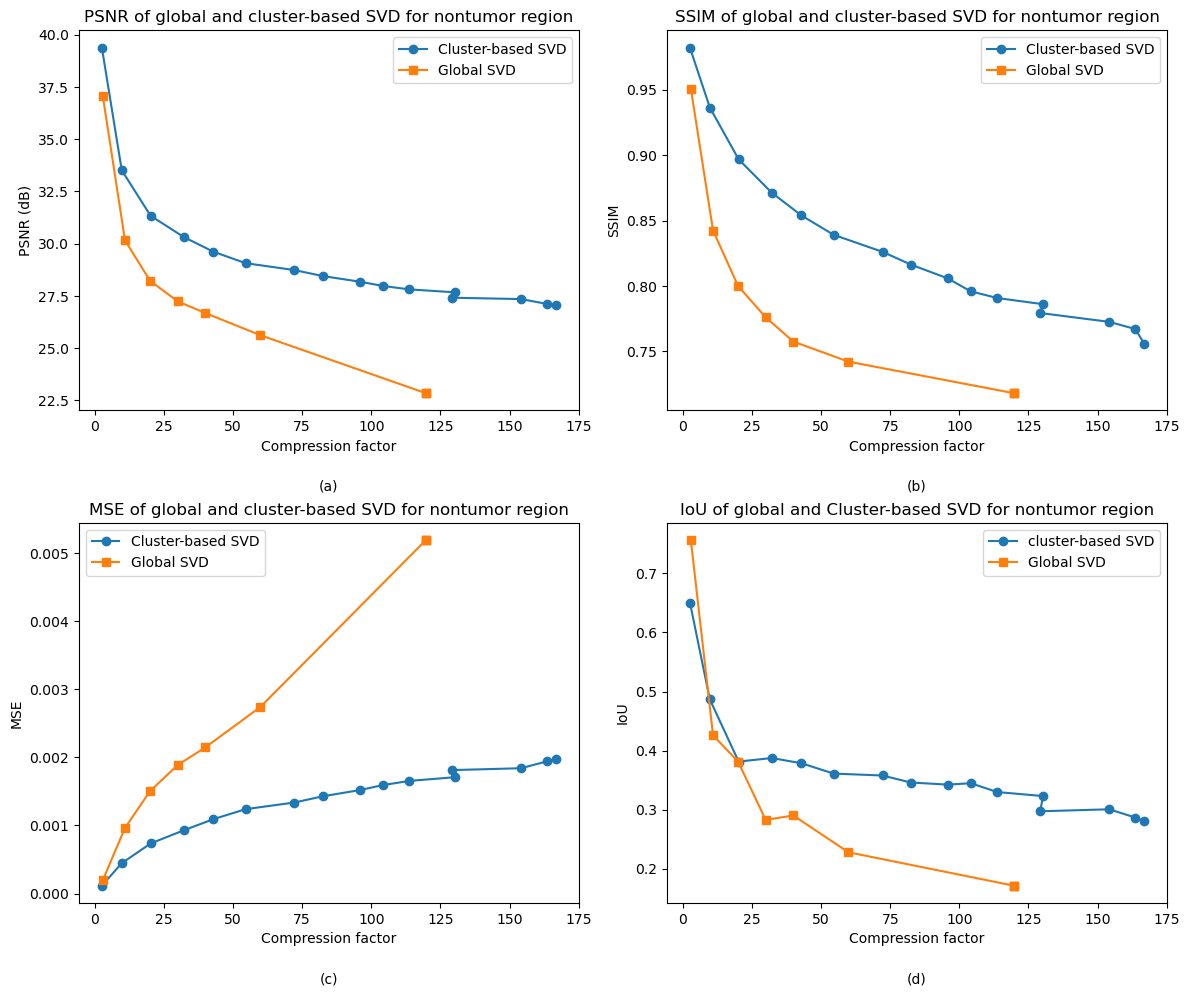}
    \caption{Performance comparison of global SVD and cluster-based SVD for the non-tumor region across different compression factors. (a) PSNR (dB) vs. Compression factor, (b) SSIM vs. Compression factor, (c) MSE vs. Compression factor, and (d) IoU vs. Compression factor.}
    \label{fig:MRI_nontumor}
\end{figure*}

Medical images contain critical structural details that must be preserved during compression. Fig.~\ref{fig:edge detection} shows a comparison in edge preservation for global and cluster-based SVD. Fig.~\ref{fig:edge detection}(a) displays the original medical image. Fig.~\ref{fig:edge detection}(b) shows the edge map generated directly from the original image, highlighting the inherent structural details. Fig.~\ref{fig:edge detection}(c) presents the edge map obtained from the image reconstructed using global SVD, while Fig.~\ref{fig:edge detection}(d) displays the edge map from the cluster‐based SVD reconstruction. Figs.~\ref{fig:edge detection}(e) and~\ref{fig:edge detection}(f) illustrate the corresponding edge loss maps for the global and cluster‐based SVD methods, respectively. The results demonstrates that cluster-based SVD outperforms global SVD in preserving edge details during image reconstruction. The edges extracted from the cluster-based SVD-reconstructed image exhibit greater similarity to those in the original image, while global SVD shows some degradation and distortion. This can bee seen in the edge loss maps, where global SVD (Fig.~\ref{fig:edge detection}e) exhibits more widespread loss, whereas cluster-based SVD (Fig.~\ref{fig:edge detection}f) has a more localized and less intense loss pattern. This suggests that cluster-based SVD maintains finer structural details more effectively. Furthermore, in terms of the EPI scores, the cluster-based SVD achieves 0.7466, compared to 0.6705 for global SVD.
\begin{figure}
    \centering
    \includegraphics[width=0.9\linewidth]{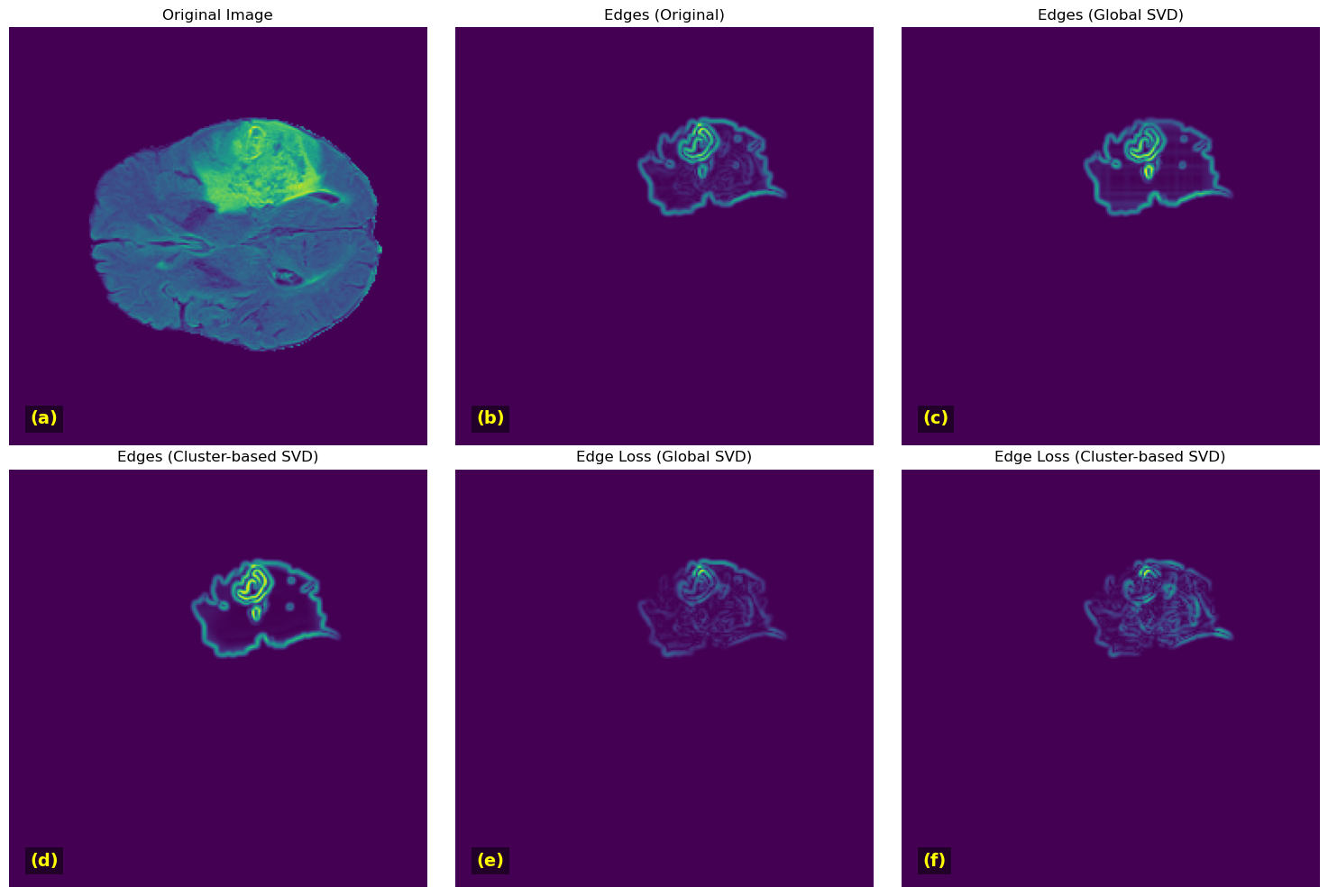}
    \caption{Comparison of edge preservation in global and cluster-based SVD. (a) Original medical image, (b) edges detected from the original image, (c) edges from global SVD-reconstructed image, (d) edges from cluster-based SVD-reconstructed image, (e) edge loss map for global SVD, and (f) edge loss map for cluster-based SVD.}
    \label{fig:edge detection}
\end{figure}

Global SVD has a lower computational cost due to its simpler uniform compression strategy. However, as discussed, this approach is limited in its ability to preserve fine details, especially in the tumor region, and struggles to achieve higher compression without significant image degradation. On the other hand, cluster-based SVD incurs higher computational costs as the compression factor increases. This increase is due to the additional step of clustering and the independent SVD computations performed within each cluster as discussed in Section~\ref{sec:computational cost}. Despite these additional costs, cluster-based SVD outperforms global SVD in preserving structural details, features and maintaining superior image quality at different compression factors.
\section{Discussion and Limitations}{\label{sec:discussion}}
The results of our study demonstrate that the cluster‐based SVD approach provides notable improvements over the global SVD method. By exploiting the local and non-local self-similarity in medical images, cluster-based SVD can group similar patches and perform localized SVD, which allows for adaptive compression. This adaptation enables the method to preserve important structural details such as edges and localized textures, which are often critical in medical imaging. For example, in MRI scans where tumor regions demand high fidelity, the localized processing allows for better preservation of fine details, even as compression factors increase. In contrast, global SVD applies a uniform compression across the entire image, often leading to block artifacts and significant loss of local features, particularly at higher compression factors. The advantage of the cluster-based method becomes even more evident when considering modalities like ultrasound, CT scan scans, and X-ray images. Each of these modalities exhibits inherent variations in texture and intensity; ultrasound images may suffer from low contrast and speckle noise, while CT scan and X-ray images require accurate representation of subtle gradations to capture anatomical details correctly. By adjusting the number of singular values retained in each cluster based on the local energy distribution, the cluster-based approach consistently yields higher-quality reconstructions with improved metrics such as PSNR, SSIM, and reduced MSE. The edge preservation capability of cluster-based SVD is particularly significant, as it maintains clear and sharp boundaries that are important for accurate segmentation and diagnosis.

To compare the performance of cluster-based and global SVD across different modalities, we computed the relative error within the ROI and outside it, along with the SSIM error and IoU error. In additionally, the EPI error was computed specifically within the tumor region to evaluate the preservation of structural details in the area of clinical relevance. Before computing these metrics, all reconstructed and ground truth images were resized to the same dimension. Furthermore, the compression factor was kept approximately the same for all modalities to allow for fair assessment of reconstruction quality. As shown in Table~\ref{tab:relative error}, cluster-based SVD consistently outperformed global SVD across all medical image modalities. The X-ray image modality preserved and extracted important feature the most, especially at a higher compression factor, due to their homogeneous backgrounds and sharp edges. The consistently low errors in X-ray image modality suggest that cluster-based SVD is more suited for this modality. CT scan image modality benefited notably, with better preservation of subtle gradients important for diagnostic accuracy. In MRI, the presence of complex anatomical textures and high-frequency tumor regions complicates the clustering process, leading to moderate gain. For ultrasound images, both cluster-based and global SVD methods performed relatively poorly, with high relative, SSIM, and IoU errors. This is due to ultrasound's low contrast and speckle noise which reduces patch similarity.
\begin{table*}[htbp]
\centering
\captionsetup{font=small, labelfont=bf}
\caption{Comparison between cluster‐based and global SVD methods across different medical imaging modalities (MRI, ultrasound, CT scan, and X‑ray) and two randomly selected compression factors (6 and 40). For each modality and compression factor, the table reports the relative reconstruction error (Rel. Err), SSIM error, IoU error, and EPI error (all in \%) within the ROI and the existing metrics outside the ROI.}
\label{tab:relative error}
\renewcommand{\arraystretch}{1.5}
\begin{adjustbox}{max width=1.0\textwidth}
\begin{tabular}{@{}l l *{7}{c} *{7}{c}@{}}
\toprule
\textbf{CF} & \textbf{Modality} &
\multicolumn{7}{c}{\underline{\textbf{Cluster‑based SVD}}} &
\multicolumn{7}{c}{\underline{\textbf{Global SVD}}} \\
& & 
\multicolumn{4}{c}{\underline{\textbf{ROI}}} & \multicolumn{3}{c}{\underline{\textbf{Outside ROI}}} &
\multicolumn{4}{c}{\underline{\textbf{ROI}}} & \multicolumn{3}{c}{\underline{\textbf{Outside ROI}}} \\
& & 
\makecell[c]{\textbf{Rel.}\\\textbf{Err (\%)}} & 
\makecell[c]{\textbf{SSIM}\\\textbf{Err (\%)}} & 
\makecell[c]{\textbf{IoU}\\\textbf{Err (\%)}} &
\makecell[c]{\textbf{EPI}\\\textbf{Err (\%)}} &
\makecell[c]{\textbf{Rel.}\\\textbf{Err (\%)}} & 
\makecell[c]{\textbf{SSIM}\\\textbf{Err (\%)}} & 
\makecell[c]{\textbf{IoU}\\\textbf{Err (\%)}} &
\makecell[c]{\textbf{Rel.}\\\textbf{Err (\%)}} & 
\makecell[c]{\textbf{SSIM}\\\textbf{Err (\%)}} & 
\makecell[c]{\textbf{IoU}\\\textbf{Err (\%)}} &
\makecell[c]{\textbf{EPI}\\\textbf{Err (\%)}} &
\makecell[c]{\textbf{Rel.}\\\textbf{Err (\%)}} & 
\makecell[c]{\textbf{SSIM}\\\textbf{Err (\%)}} & 
\makecell[c]{\textbf{IoU}\\\textbf{Err (\%)}} \\
\midrule
\multirow{4}{*}{6}
& MRI        
  &  6.14 & 0.70 & 4.06 & 12.40  & 10.91 &  3.88 & 45.80  
  &  6.89 &  0.86 & 4.58 & 14.31  & 16.47 &  10.76 & 47.65  \\

& Ultrasound 
  & 12.14 & 2.20 & 30.43 & 15.80  &  7.61 & 5.04 & 16.77 
  & 17.50 & 5.90 & 45.90 & 30.75  & 11.63 & 12.64 & 22.94  \\

& CT scan         
  &  7.11 & 0.11 & 9.38 & 6.37   &  5.32 & 4.55 & 1.36  
  &  9.87 & 0.21 & 11.97 & 10.17  &  9.35 & 12.65 & 3.14  \\

& X‑ray      
  &  2.86 & 0.28 & 6.38 & 5.91   &  2.31 & 3.55 & 1.93  
  &  5.46 & 0.90 & 11.67 & 12.00  &  3.46 & 9.56 & 3.41 \\

\midrule
\multirow{4}{*}{40}
& MRI        
  & 10.20 & 1.89 & 6.59 & 25.95   & 20.04 & 12.07 & 61.25  
  & 15.14 & 2.38 & 13.21 & 32.37  & 30.50 & 23.18 & 70.98  \\

& Ultrasound 
  & 21.06 & 7.40 & 60.22 & 37.76  & 16.05 & 21.93 & 30.76  
  & 32.94 & 12.97 & 80.97 & 51.72 & 24.79 & 35.64 & 41.39 \\

& CT scan         
  & 13.30 & 0.40 & 17.05 & 15.00  & 10.65 & 13.02 & 3.63  
  & 23.65 & 0.72 & 32.82 & 23.82  & 22.27 & 26.98 & 13.12 \\

& X‑ray      
  &  5.20 & 0.87 & 10.52 & 12.37  &  4.59 & 10.67 & 3.81  
  & 15.35 & 2.79 & 35.66 & 27.56 & 12.29 & 29.37 & 11.90  \\
\bottomrule
\end{tabular}
\end{adjustbox}
\end{table*}

Despite these advantages, the cluster-based approach is not without limitations. The introduction of a clustering step increases the overall computational complexity, which may present challenges in real-time or resource-constrained environments. Additionally, the performance of the method is highly sensitive to key parameters such as patch size, the number of clusters, and the energy threshold used to determine \(r_{\mathbf{C}_k}\). Inappropriate parameter settings can either lead to over-compression, which results in the loss of important diagnostic details, or under-compression, which fails to achieve the desired reduction in data storage size.
\section{Conclusion and future work}{\label{sec:conclusion}}
In this work, we introduced a cluster-based SVD framework to address the challenges of compressing large-scale medical images while preserving and extracting critical structural details. Our study comprehensively compares global and cluster-based SVD. These methods were applied to medical images from different modalities, such as MRI, CT scan, ultrasound, and X-ray, to evaluate their effectiveness in preserving and extracting important features. Our analysis showed that the uniform compression enforced by global SVD fails to adapt to the local structural complexity of medical images. This often leads to visible artifacts and the loss of diagnostically important features, especially in heterogeneous regions such as tumors or lesions. In contrast, the cluster-based SVD method exploits local and non-local self-similarity to apply adaptive compression. This approach consistently improves image fidelity across modalities, enabling better preservation of fine structural details, as shown by higher PSNR, SSIM, IoU, and EPI values. However, our experimental results also highlight that the performance of cluster-based SVD is modality-dependent. Cluster-based SVD performed best in X-ray image modality, where the homogeneous background and sharp edges allowed for effective grouping and compression. On the other hand, it performed poorly on ultrasound image modality, where the presence of speckle noise and low contrast significantly reduced patch similarity, reducing the ability of cluster-based SVD to exploit localized low-rank structures effectively. Additionally, we provided a theoretical framework to quantify the compression factor. We showed how the overall compression factor is influenced by parameters such as patch size, \(K\) and \(\alpha\). Moreover, we demonstrated that there is not always a general global rank \(r_g\) for which the cluster-based and global SVD compression factors are approximately the same. While cluster-based SVD has higher computational complexity than global SVD, the study demonstrated that these costs could be mitigated through parallel processing, making the approach feasible for large-scale medical images.

Future work will investigate alternative clustering methods that are less computationally expensive than k-means. While k-means is effective in grouping structurally similar patches, its iterative nature and sensitivity to initialization can lead to increased processing time, especially in high-resolution medical images. Exploring faster clustering techniques may help reduce the overall computational burden while maintaining or improving clustering quality. Exploring clustering techniques, such as mini-batch k-means or hierarchical clustering with early stopping, may offer speed-ups without significantly compromising performance. Such improvements could make the proposed method more scalable, especially for high-resolution 3D medical images and real-time clinical applications.

\printcredits
\section*{Declaration of competing interest}
The authors declare that they have no known competing financial interests or personal relationships that could have appeared to influence the work reported in this paper.
\section*{Acknowledgments}
This research was supported in part by the Department of Higher Education and Training (DHET) through the University Staff Doctoral Programme (USDP) and in part by the National Research Foundation of South Africa (Ref No. CSRP23040990793).
\bibliographystyle{cas-model2-names}

\bibliography{cas-refs}

\appendix
\section{Appendix}
In this appendix, we discuss the application of cluster-based and global SVD techniques to different medical imaging modalities, such as MRI, CT scans, X-rays, and ultrasound, to examine their effectiveness in  compression.
\subsection{Experiment 1: MRI}
This section presents a detailed performance comparison between global SVD and cluster‐based SVD for an MRI image in Fig.~5(i) of the main document. Fig.~\ref{fig:edge} compares edge preservation in global and cluster-based SVD reconstructions. Fig.~\ref{fig:edge}(a) presents the original medical image, while Fig.~\ref{fig:edge}(b) displays the edge map extracted from it, highlighting the natural structural details. Fig.~\ref{fig:edge}(c) shows the edge map obtained from the image reconstructed using global SVD, showing that many fine edges have been lost or blurred. In contrast, Fig.~\ref{fig:edge}(d) presents the edge map from the cluster-based SVD reconstruction, where edges are sharper and more consistent with the original image. The edge loss maps in Figs.~\ref{fig:edge}(e) and (f) further emphasize this difference. Fig.~\ref{fig:edge}(e) shows widespread edge loss throughout the image, whereas Fig.~\ref{fig:edge}(f) demonstrates more localized and less severe loss, indicating better preservation of structural details.
\begin{figure}
    \centering
    \includegraphics[width=1.0\linewidth]{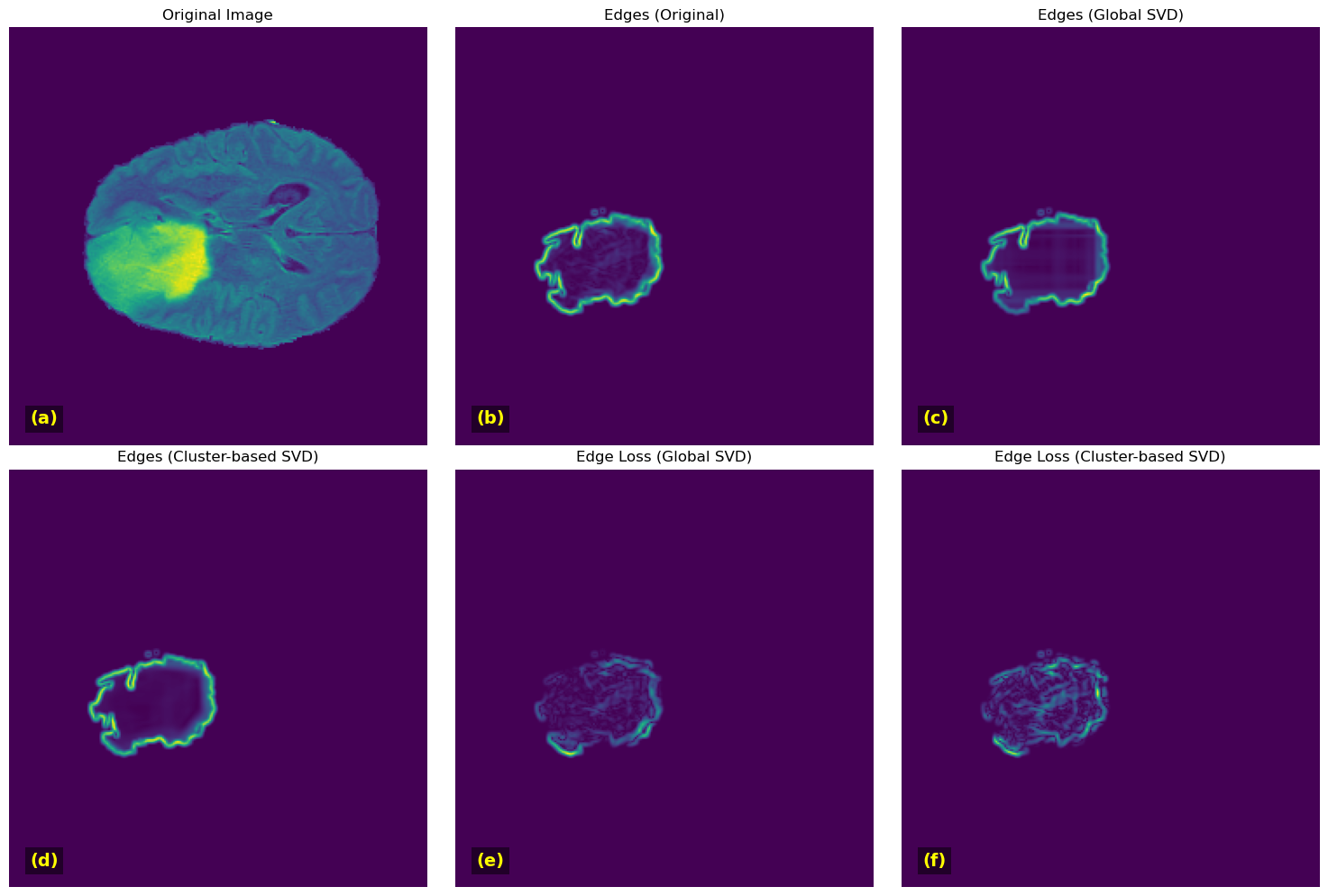}
    \caption{Comparison of edge preservation in global and cluster-based SVD reconstructed images.}
    \label{fig:edge}
\end{figure}
While Fig.~\ref{fig:tumor} shows that in the tumor region, cluster‐based SVD consistently outperforms global SVD as the compression factor increases; specifically, it achieves higher PSNR and SSIM values, lower MSE, and a higher IoU, indicating that cluster‐based SVD is much more effective at preserving fine details and extracting extracting important features during aggressive compression.
\begin{figure*}
    \centering
    \includegraphics[width=1.0\linewidth]{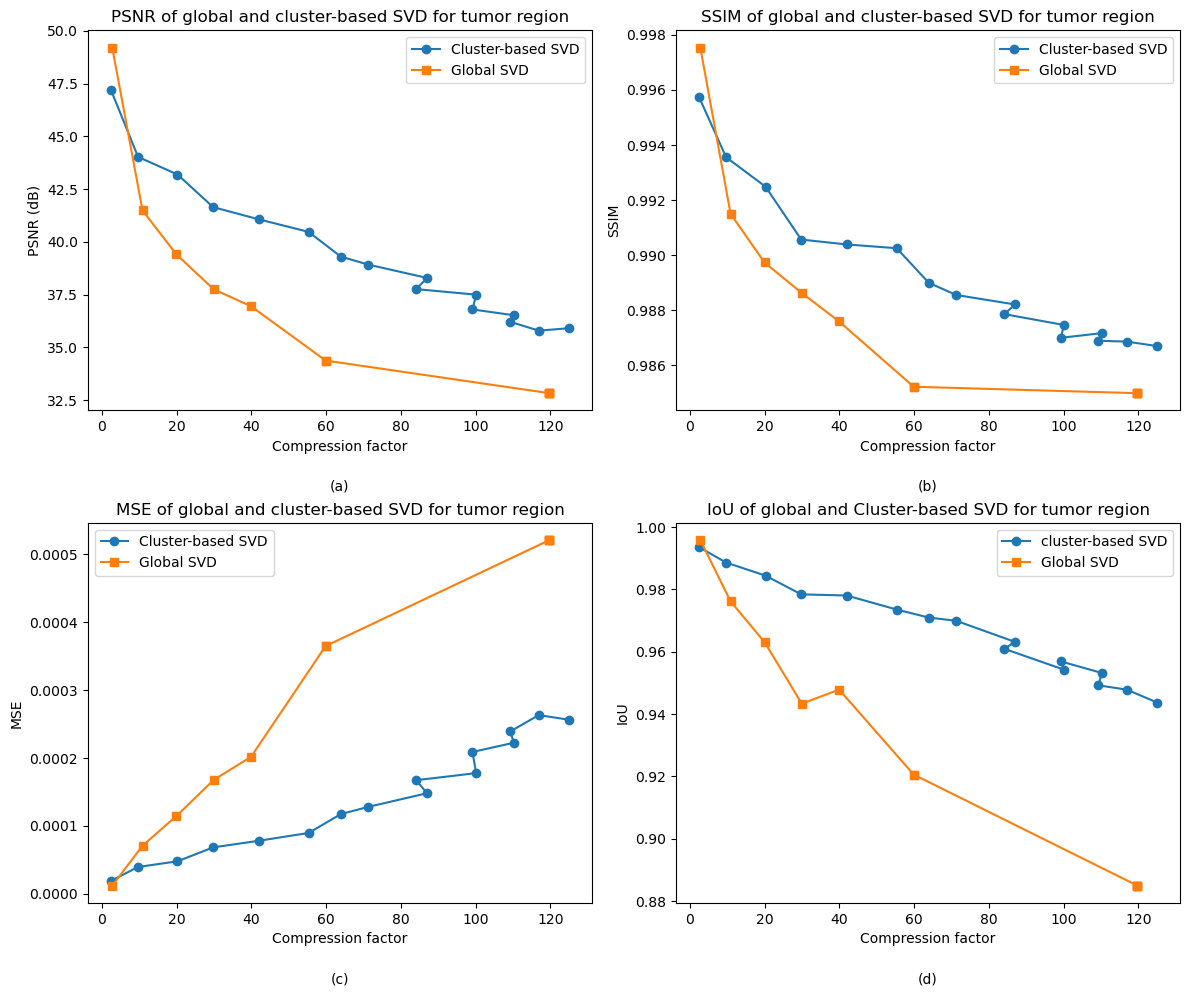}
    \caption{Performance comparison of global SVD and cluster-based SVD for the tumor region across different compression factors. (a) PSNR (dB) vs. Compression factor, (b) SSIM vs. Compression factor, (c) MSE vs. Compression factor, and (d) IoU vs. Compression factor.}
    \label{fig:tumor}
\end{figure*}
 Fig.~\ref{fig:nontumor} shows that in the non-tumor areas, a similar trend is observed. Cluster‐based SVD maintains superior image quality by preserving important features better, allowing for more aggressive compression in non-critical regions without compromising overall diagnostic quality. Cluster-based SVD achieves significantly better reconstruction quality despite higher computational costs. 
\begin{figure*}
    \centering
    \includegraphics[width=1.0\linewidth]{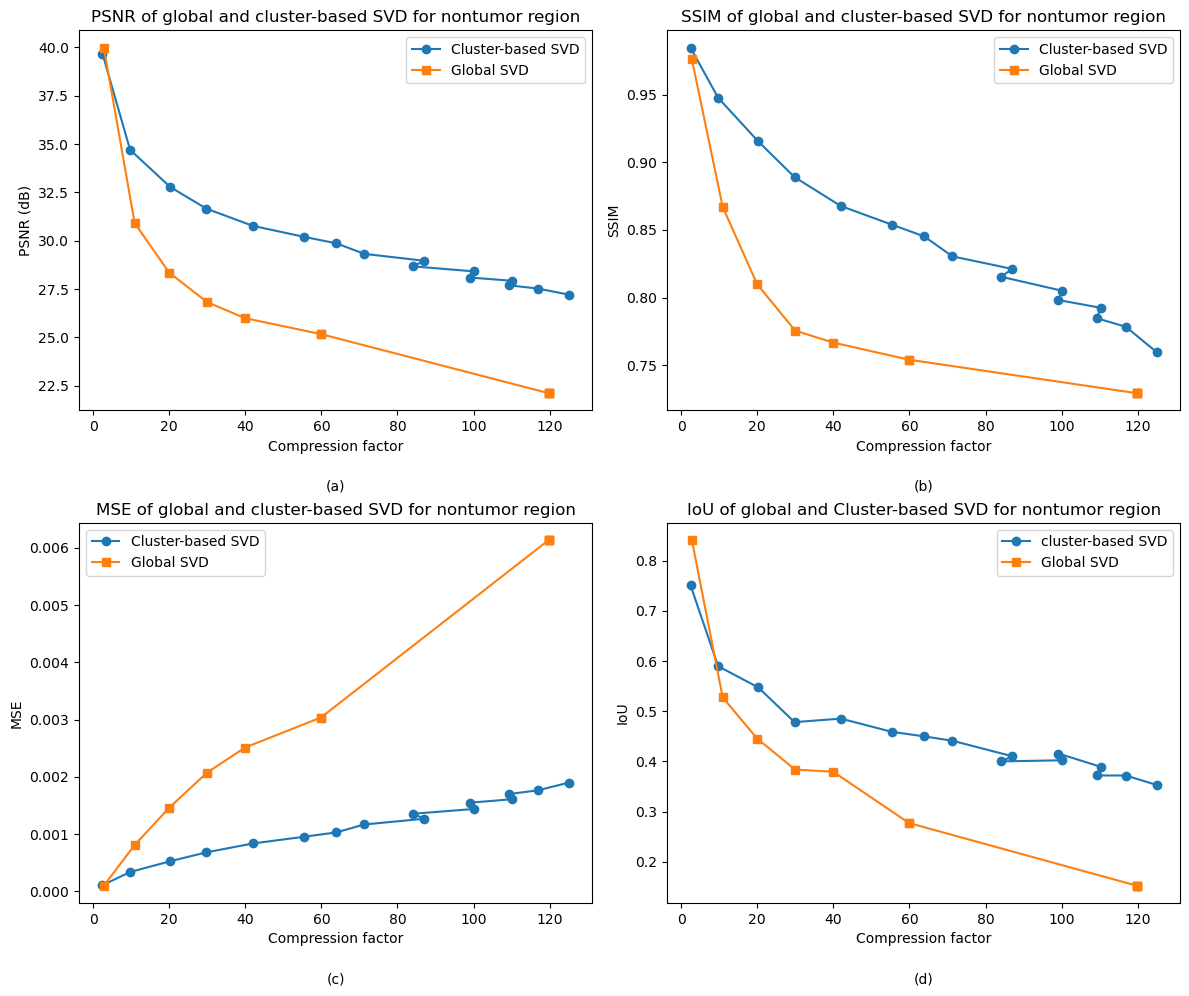}
    \caption{Performance comparison of global SVD and cluster-based SVD for the non-tumor region across different compression factors. (a) PSNR (dB) vs. Compression factor, (b) SSIM vs. Compression factor, (c) MSE vs. Compression factor, and (d) IoU vs. Compression factor.}
    \label{fig:nontumor}
\end{figure*}
\subsection{Experiment 2: Ultrasound}
\subsubsection{Dataset Description}
The breast ultrasound images dataset is a publicly available dataset on \href{https://www.kaggle.com/datasets/aryashah2k/breast-ultrasound-images-dataset}{Kaggle} that contains 780 ultrasound images collected for breast cancer detection \cite{al2020dataset}. The images are categorized into three classes: normal, benign, and malignant, with expert annotations indicating the presence and type of lesions. Each image in the dataset is stored in PNG format and varies in size. Alongside the images, the dataset includes mask annotations that highlight the ROI, allowing for precise segmentation and classification tasks. The images were acquired using standard ultrasound imaging protocols. This dataset is particularly relevant for evaluating low-rank approximation methods due to the presence of speckle noise, a common challenge in ultrasound imaging. The structural properties of these images make them suitable for assessing denoising and compression techniques while ensuring the preservation of diagnostically significant details. In this experiment, we randomly selected one image to analyze the effects of our method on ultrasound modality.
\subsubsection{Performance Evaluation}
The results presented in Figs.~\ref{fig:cancer_image} to \ref{fig:non_cancer_region} demonstrate the advantage of cluster-based SVD over global SVD for medical image compression, particularly in preserving and extracting critical features. Figure 10 compares the two methods at a compression factor of approximately 65, showing that while global SVD introduces block artifacts and loss of fine details in both cancerous and non-cancerous regions (Fig.~\ref{fig:cancer_image}(c)), cluster-based SVD maintains superior image quality with fewer artifacts (Fig.~\ref{fig:cancer_image}(d)). This is further evidenced by the residual error maps (Fig.~\ref{fig:cancer_image}(e) to (h)), where cluster-based SVD exhibits significantly lower errors in both cancer and non-cancer areas (Fig.~\ref{fig:cancer_image}(f) and (h)). 
\begin{figure}
    \centering
    \includegraphics[width=1.0\linewidth]{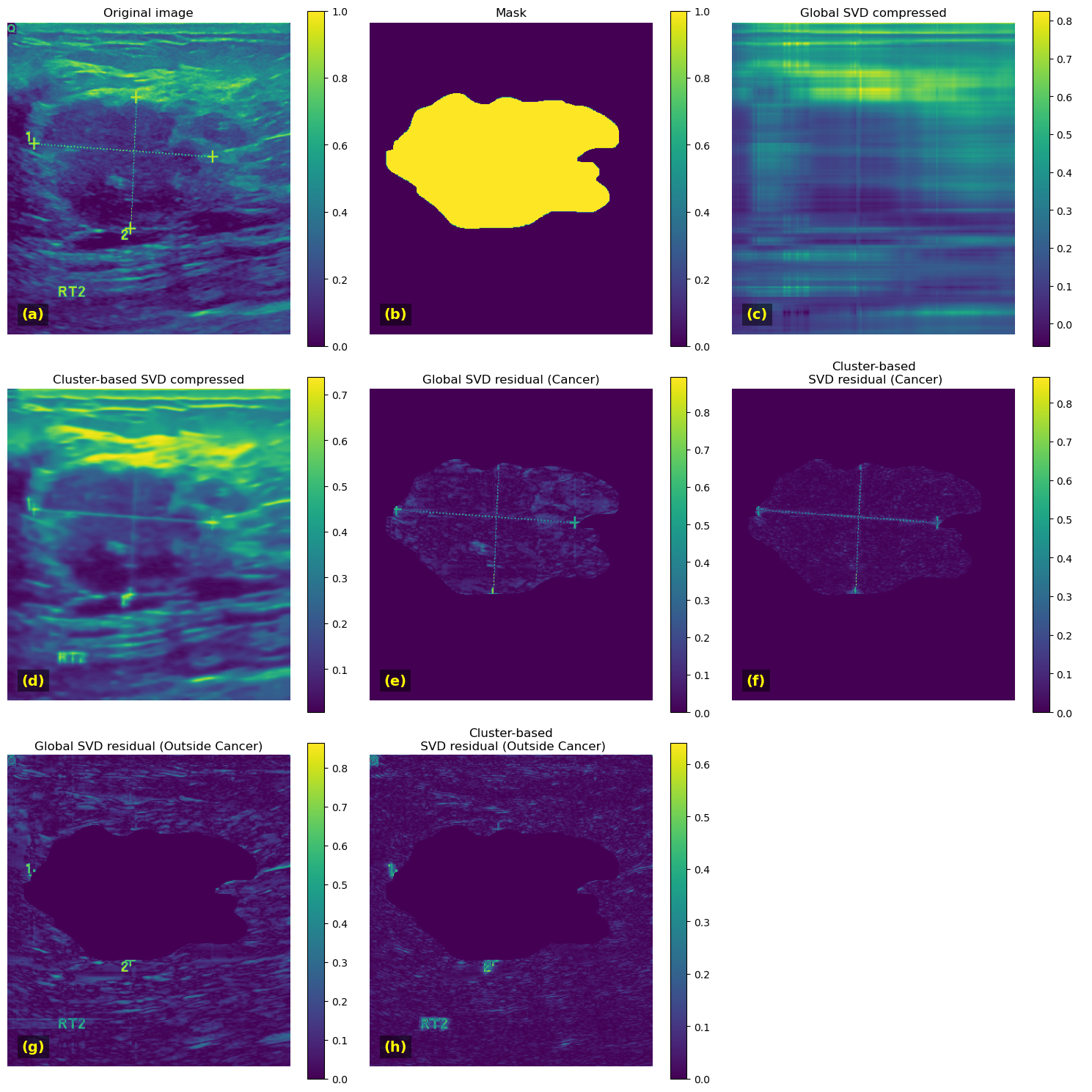}
    \caption{This figure compares the effects of global and cluster-based SVD compression on a medical image at a compression factor of approximately 65. (a) Shows the original image, while (b) presents the binary mask highlighting the cancerous region. (c) and (d) show the compressed images obtained using global SVD and cluster-based SVD, respectively. (e) and (f) illustrate the residual errors for global and cluster-based SVD within the cancerous region, whereas (g) and (h) depict the residual errors outside the cancerous region.}
    \label{fig:cancer_image}
\end{figure}
The edge maps in Fig.~\ref{fig:edge_cancer} reveal that global SVD tends to blur or lose fine edges (Figs.~\ref{fig:edge_cancer}(c) and (e)), whereas cluster-based SVD retains sharper, more accurate edges closer to the original image (Figs.~\ref{fig:edge_cancer}(d) and (f)). The EPI score further quantifies this advantage, with global SVD achieving 0.3915 and cluster-based SVD achieving 0.5300, indicating better structural integrity.
\begin{figure}
    \centering
    \includegraphics[width=1.0\linewidth]{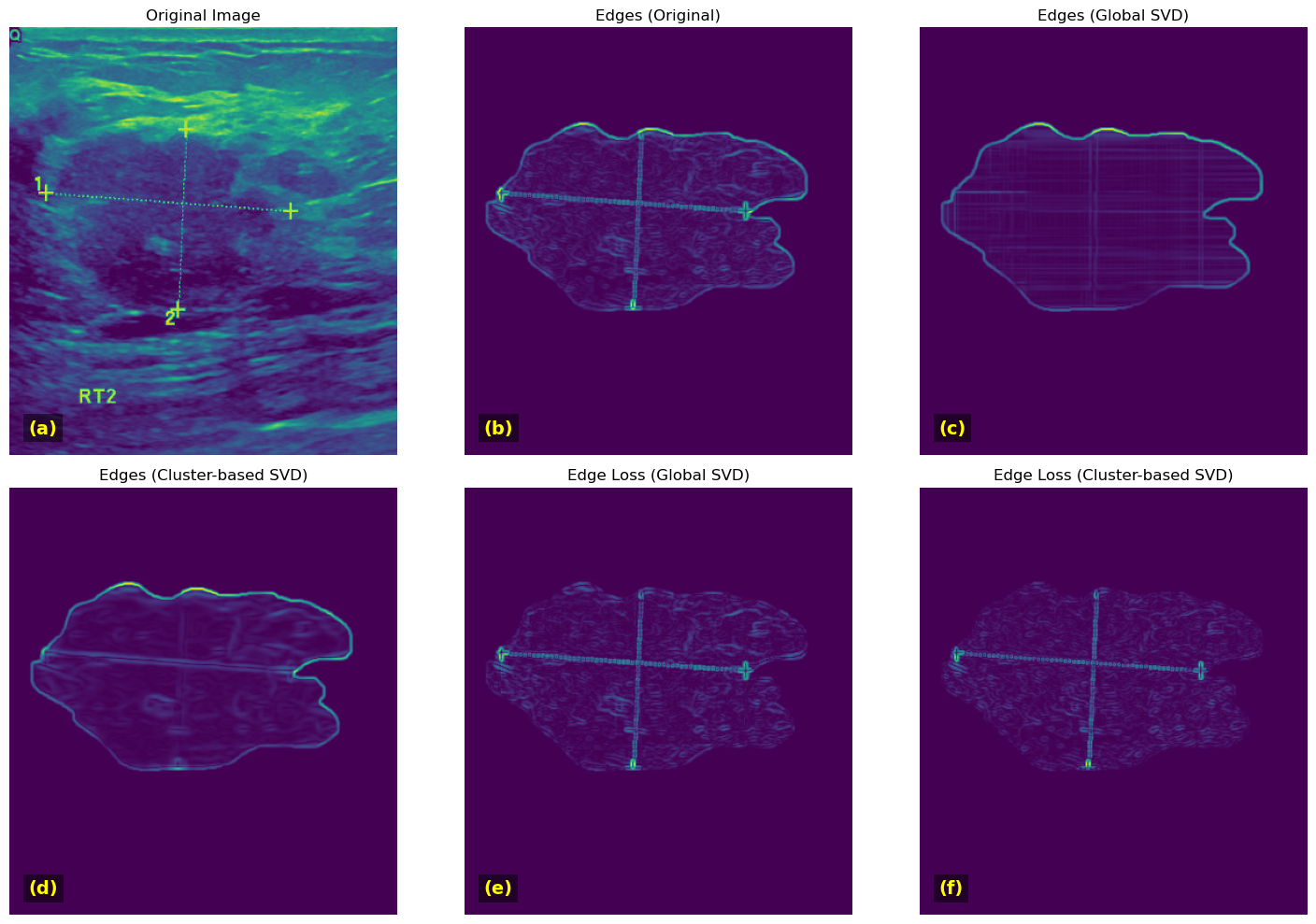}
    \caption{Comparison of edge preservation in global and cluster-based SVD. (a) shows the original image. (b) presents the detected edges in the original image. (c) and (d) illustrate the edges extracted from global and cluster-based SVD-reconstructed images, respectively, whereas (e) and (f) depict the corresponding edge loss for each method.}
    \label{fig:edge_cancer}
\end{figure}
Figs.~\ref{fig:cancer_region} and \ref{fig:non_cancer_region} evaluate the performance of global and cluster-based SVD across varying compression factors for cancerous and non-cancerous regions using different performance metrics, respectively. Cluster-based SVD consistently outperforms global SVD in PSNR, SSIM, MSE, and IoU. Notably, cluster-based SVD allows for more aggressive compression in non-cancerous regions without degrading quality in critical areas, optimizing storage efficiency while preserving diagnostic relevance.
\begin{figure*}[!t]
    \centering
    \includegraphics[width=1.0\linewidth]{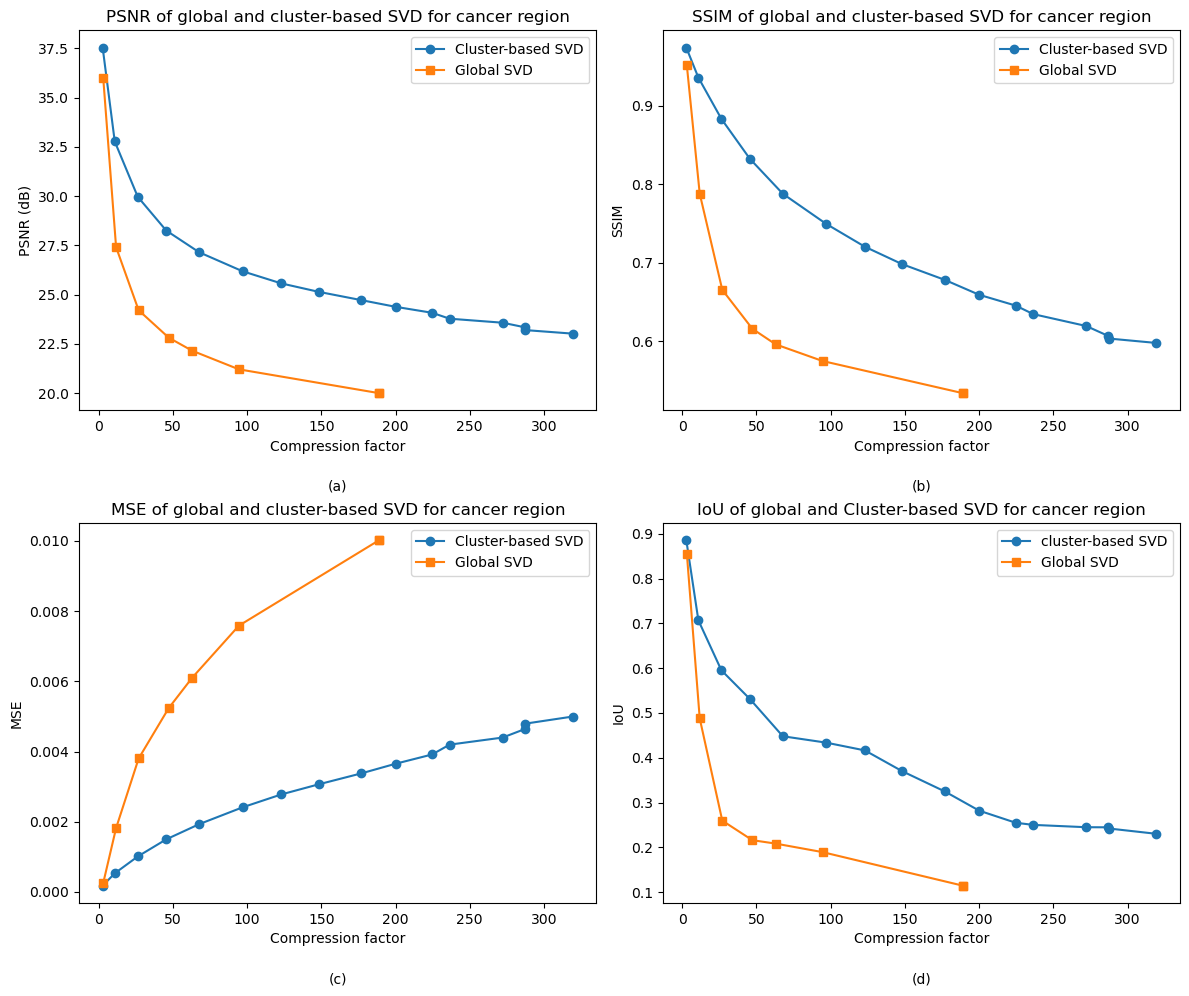}
    \caption{Performance comparison of global and cluster-based SVD for cancer region preservation across different compression factors. (a) shows PSNR, (b) shows SSIM, (c) shows MSE, and (d) show IoU.}
    \label{fig:cancer_region}
\end{figure*}
\begin{figure*}
    \centering
    \includegraphics[width=1.0\linewidth]{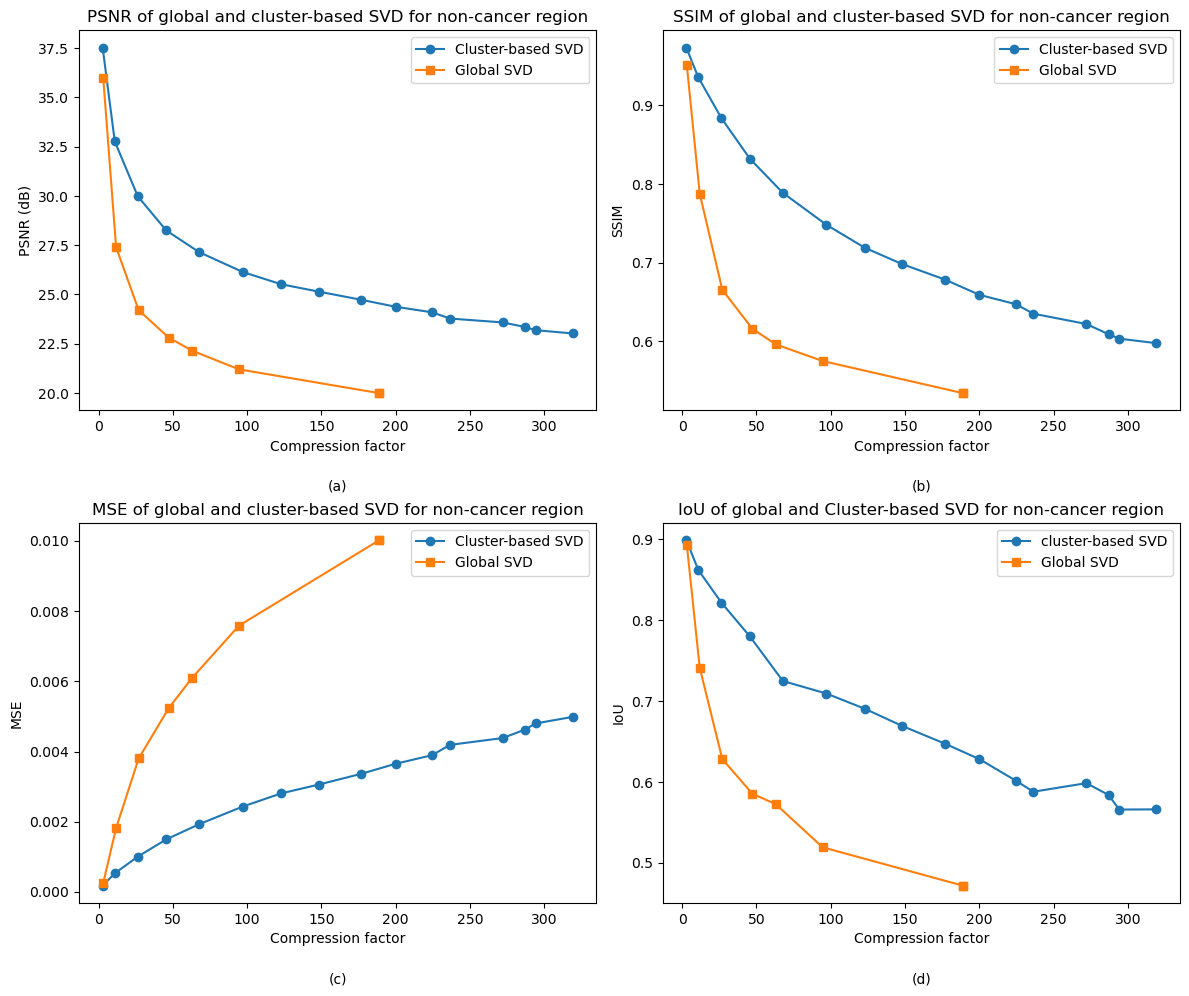}
    \caption{Performance comparison of global and cluster-based SVD for a non-cancer region preservation across different compression factors. (a) shows PSNR, (b) shows SSIM, (c) shows MSE, and (d) show IoU.}
    \label{fig:non_cancer_region}
\end{figure*}
\subsection{Experiment 3: CT scan}
\subsubsection{Dataset Description}  
The COVID-19 CT scan lesion segmentation dataset is a publicly available dataset on \href{https://www.kaggle.com/datasets/maedemaftouni/covid19-ct-scan-lesion-segmentation-dataset}{Kaggle} that contains 2\,729 lung CT scan images specifically curated for COVID-19 lesion segmentation \cite{maftouni2021covidctlesion}. The dataset includes a collection of CT slices with the corresponding lesion masks, which provide pixel-wise annotations of infected regions.  These annotations facilitate tasks such as automated lesion detection, segmentation, and severity assessment of COVID-19 infections. The scans were acquired from real patients diagnosed with COVID-19. Given the challenges associated with CT imaging, such as inherent noise and contrast variability, this dataset offers an excellent opportunity to assess image compression techniques. In this experiment, we randomly selected one image to analyze the effects of our method on CT modality.
\subsubsection{Performance Evaluation}
Fig.~\ref{fig:image_ct} compares the two methods on a CT scan at a compression factor of approximately 51. Fig.~\ref{fig:image_ct}(a) shows the original CT image, while Fig.~\ref{fig:image_ct}(b) shows a binary mask that highlights the infected (COVID‐19) areas. The globally compressed image in Fig.~\ref{fig:image_ct}(c) exhibits block artifacts and noticeable degradation, and in Fig.~\ref{fig:image_ct}(d) the cluster‐based SVD image has fewer artifacts. The residual error maps in Figs.~\ref{fig:image_ct}(e) and (f) (within the infection region) and in Figs.~\ref{fig:image_ct}(g) and (h) (outside the infection region) show that global SVD produces widespread errors, whereas cluster‐based SVD limits these errors to more localized areas, demonstrating its superior ability to preserve and extract critical image features.
\begin{figure}
    \centering
    \includegraphics[width=1.0\linewidth]{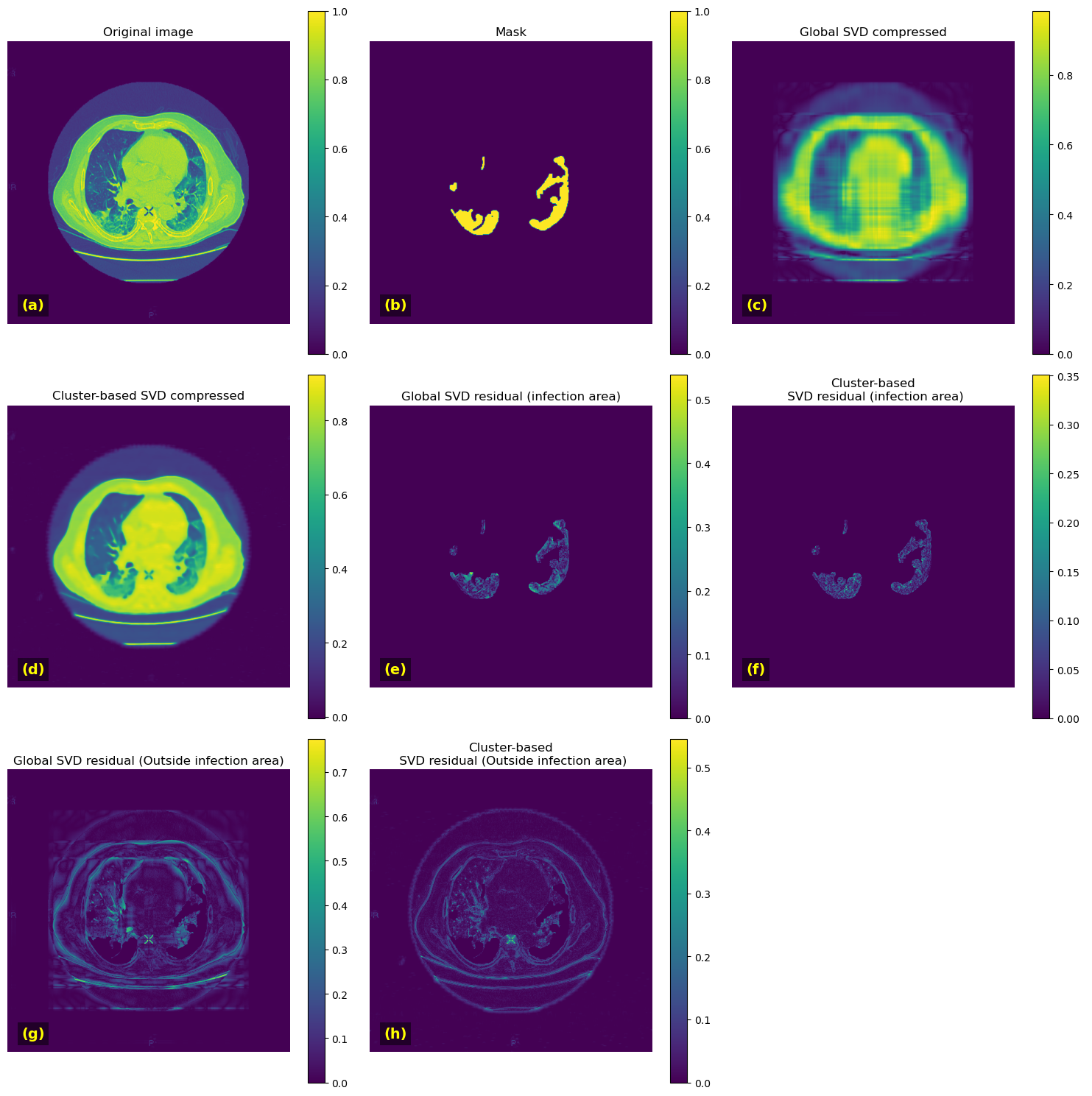}
    \caption{The figure compares cluster-based SVD and global SVD compression techniques applied to a CT scan image, focusing on COVID-19 infected and non-infected areas at a compression factor of approximately 51. (a) shows the original CT scan image. (b) presents a binary mask highlighting the infection regions. (c) displays the globally compressed image. (d) shows the cluster-based SVD compressed image. (e) and (f) illustrate the residual differences (error maps) within the infection area for global and cluster-based SVD, respectively. (g) and (h) depict the residual differences outside the infection area for global SVD and cluster-based SVD, respectively.}
    \label{fig:image_ct}
\end{figure}
On the other hand, Fig.~\ref{fig:edge_ct} focuses on edge preservation in the CT image. Fig.~\ref{fig:edge_ct}(a) shows the original image and its corresponding edge map in Fig.~\ref{fig:edge_ct}(b). The edge maps from the compressed images in Fig.~\ref{fig:edge_ct}(c) for global SVD and Fig.~\ref{fig:edge_ct}(d) for cluster‐based SVD show that global SVD tends to blur or lose fine edges, while cluster‐based SVD retains sharper and more accurate edges. This is further quantified by the edge loss maps in Fig.~\ref{fig:edge_ct}(e) and Fig.~\ref{fig:edge_ct}(f), where the global method displays more extensive loss compared to the less severe loss in the cluster‐based approach. The EPI score further quantifies this advantage, with global SVD achieving
0.7029 and cluster-based SVD achieving 0.7232, indicating better structural
integrity.
\begin{figure}
    \centering
    \includegraphics[width=1.0\linewidth]{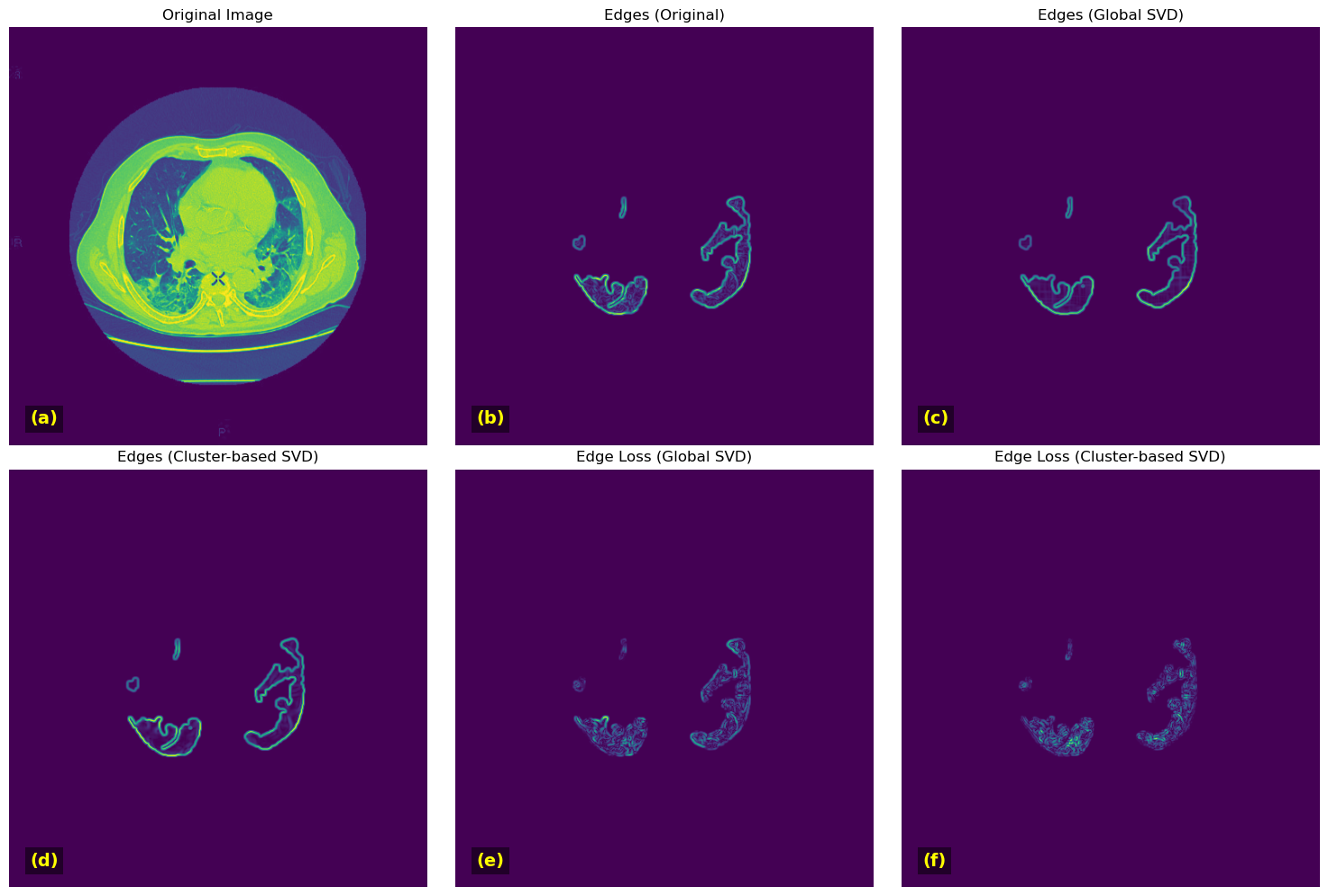}
    \caption{The figure compares the effects of global SVD and cluster-based SVD on edge preservation in a CT scan image, with a focus on infection areas. (a) shows the original CT scan image. (b) presents the detected edges in the original image. (c) and (d) display the edges detected in the globally compressed and cluster-based compressed images, respectively. (e) and (f) depict the edge loss maps for global SVD and cluster-based SVD, respectively.}
    \label{fig:edge_ct}
\end{figure}
While Figs.~\ref{fig:infection_ct} and \ref{fig:noninfection_ct} evaluate the performance of global and cluster-based SVD across varying compression factors for infected (COVID-19) and non-infected regions using different performance metrics, respectively. Cluster-based SVD consistently outperforms global SVD in PSNR, SSIM, MSE, and IoU. Notably, cluster-based SVD ensures superior preservation and extraction of critical features in infected regions while allowing more aggressive compression in non-infected areas. 
\begin{figure*}
    \centering
    \includegraphics[width=1.0\linewidth]{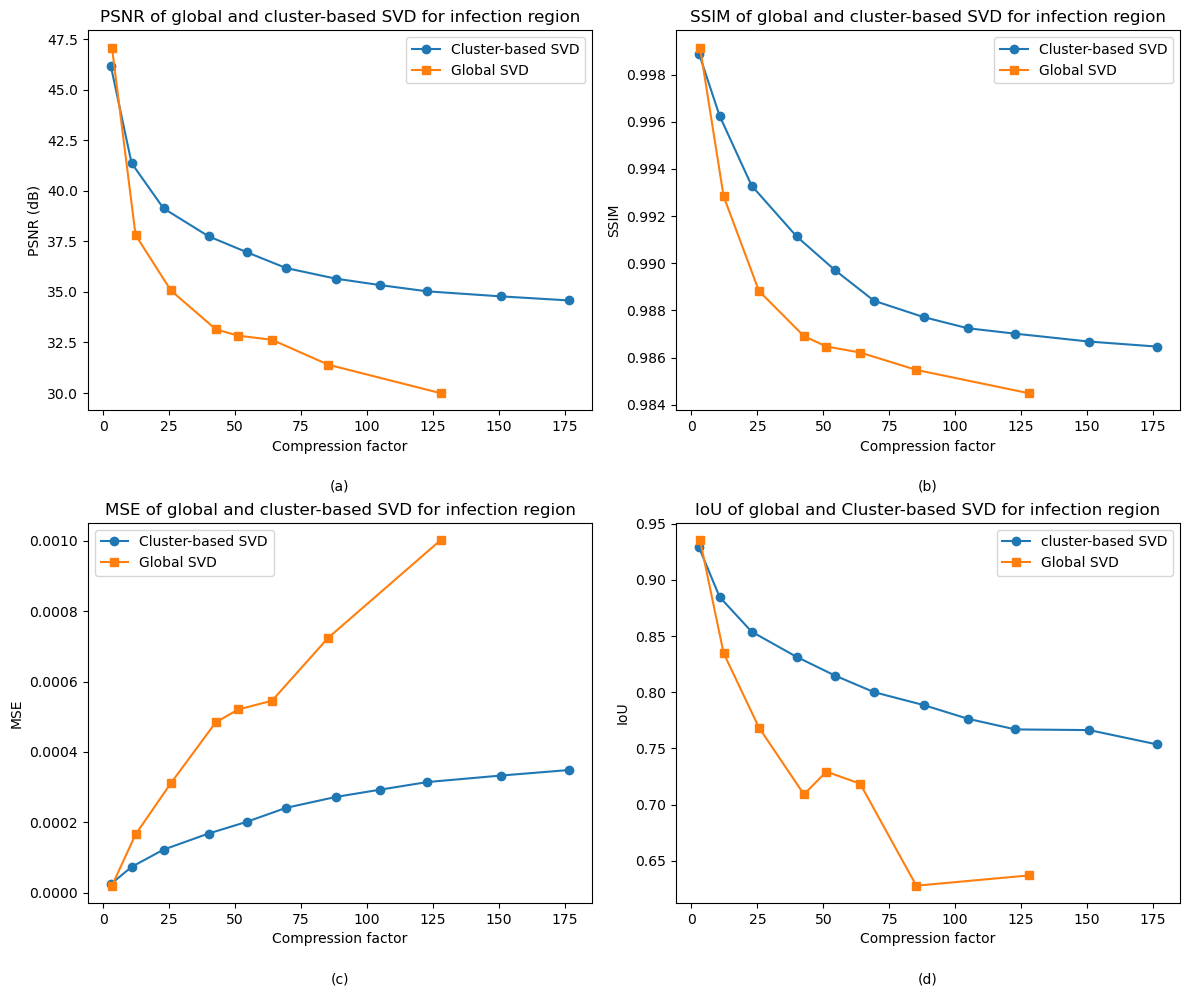}
    \caption{Performance comparison of global and cluster-based SVD for COVID-19 infection region preservation across different compression factors. (a) shows PSNR, (b) shows SSIM, (c) shows MSE, and (d) show IoU.}
    \label{fig:infection_ct}
\end{figure*}
\begin{figure*}
    \centering
    \includegraphics[width=1.0\linewidth]{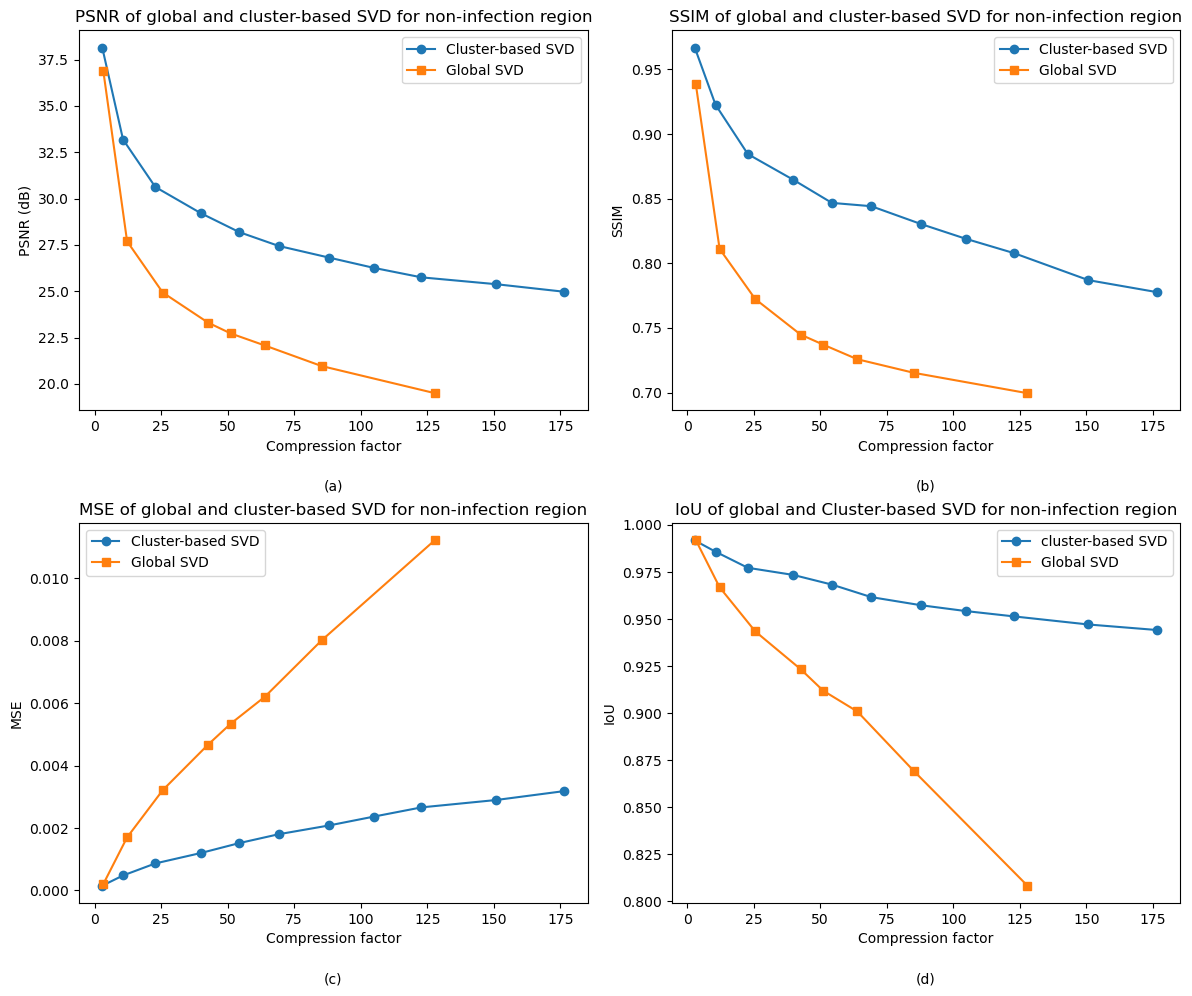}
    \caption{Performance comparison of global and cluster-based SVD for non-COVID-19 infected region preservation across different compression factors. (a) shows PSNR, (b) shows SSIM, (c) shows MSE, and (d) shows IoU.}
    \label{fig:noninfection_ct}
\end{figure*}
\subsection{Experiment 4: X-ray}
\subsubsection{Dataset Description}
The COVID-QU-Ex dataset is a publicly available collection of 33,920 chest X-ray (CXR) images hosted on \href{https://www.kaggle.com/datasets/anasmohammedtahir/covidqu}{Kaggle} and curated by researchers at Qatar University \cite{tahir2021covidquex}. It is designed to facilitate research in COVID-19 detection and related thoracic diseases using medical imaging. The dataset is composed of three categories: 11,956 images from COVID-19 positive patients, 11,263 images labeled as normal, and 10,701 images corresponding to other types of pneumonia. All images are stored in a consistent digital format and are accompanied by labels, making the dataset suitable for tasks such as classification, segmentation, and anomaly detection. The images were acquired using standard clinical imaging protocols. Given the challenges associated with chest X-ray imaging, such as inherent noise and low contrast between anatomical structures, this dataset offers an excellent opportunity to assess image compression techniques. In this experiment, we randomly selected one image to analyze the effects of our method on the chest X-ray modality.

\subsubsection{Performance Evaluation}
The results presented in Figs.~\ref{fig:image_xray} to \ref{fig:noninfection_xray} show the advantage of cluster‐based SVD over global SVD for chest X‑ray image compression, particularly in preserving and extracting critical features. Fig.~\ref{fig:image_xray} compares the two methods at a compression factor of approximately 7, showing that while global SVD introduces noticeable block artifacts and a loss of fine structural detail (Fig.~\ref{fig:image_xray}(c)), especially in the COVID infected region (Fig.~\ref{fig:image_xray}(e)), cluster‐based SVD maintains superior image quality with fewer artifacts (Fig.~\ref{fig:image_xray}(d)). However, Fig.~\ref{fig:image_xray}(g), the residual error in the non-infection region is lower for the global SVD reconstruction than for the cluster‐based SVD approach, which is shown in Fig.~\ref{fig:image_xray}(g). This suggests that while cluster‐based SVD can effectively maintain and extract important features in the regions of clinical importance, global SVD can perform better in non-critical regions where preserving subtle details is less demanding when the compression factor is low.
\begin{figure}
    \centering
    \includegraphics[width=1.0\linewidth]{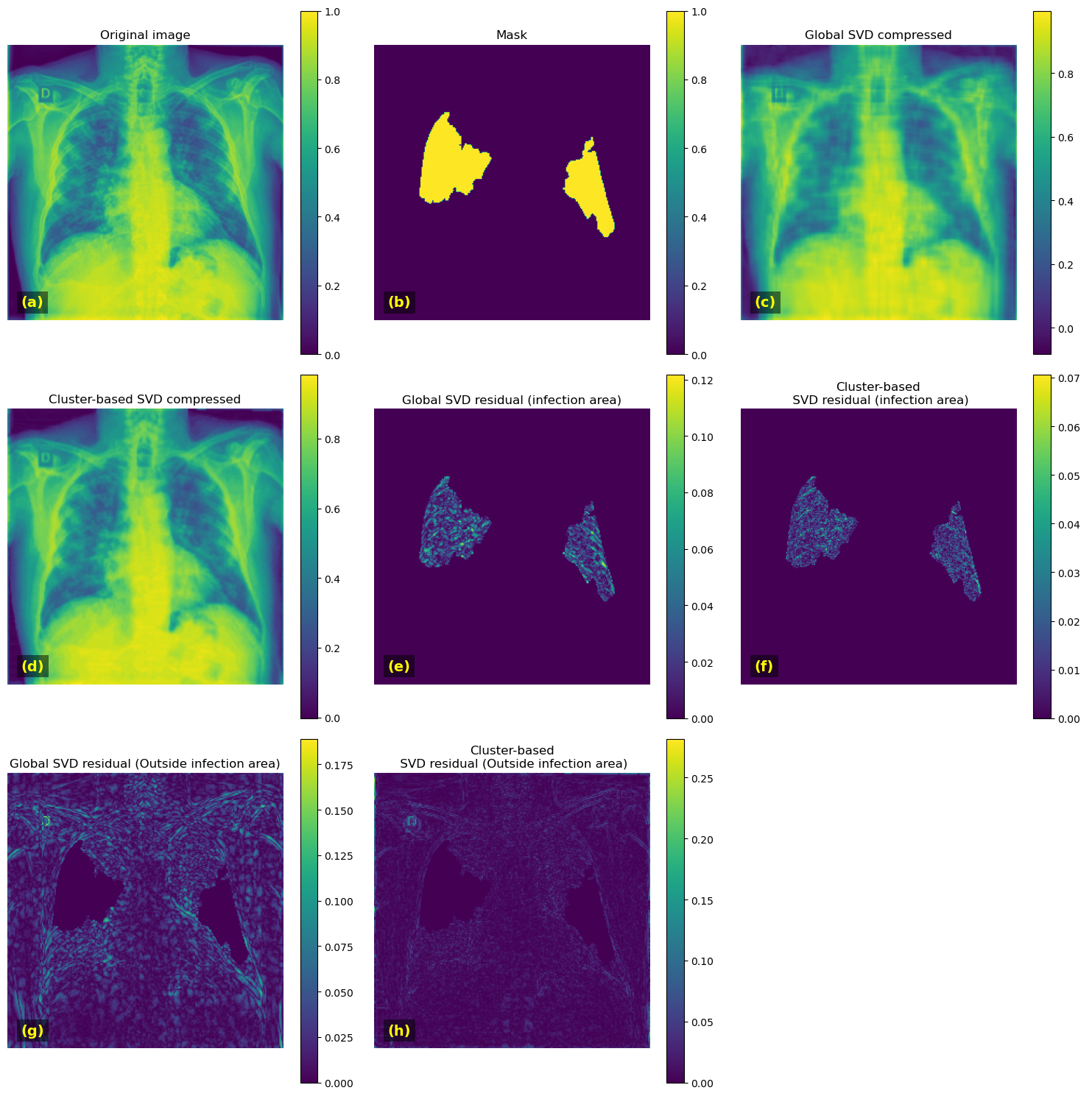}
    \caption{The figure compares cluster-based SVD and global SVD compression techniques applied to a chest X-ray image at the compression factor of approximately 7 in the infection and non-infection area. (a) shows the original X-ray image. (b) presents a binary mask highlighting the infection areas. (c) displays the image compressed using global SVD, while (d) shows the result of cluster-based SVD compression. (e) and (f) illustrate the residuals in the infection area for global and cluster-based SVD, respectively. (g) and (h) show the residuals outside the infection area for global and cluster-based SVD, respectively.}
    \label{fig:image_xray}
\end{figure}
On the other hand, Fig.~\ref{fig:edge_xray} focuses on edge preservation in the compressed chest X-ray image. In Fig.~\ref{fig:edge_xray}(a), the original image is shown alongside its corresponding edge map in Fig.~\ref{fig:edge_xray}(b), highlighting the inherent structural details. The edge maps from the compressed images in Fig.~\ref{fig:edge_xray}(c) for global SVD and Fig.~\ref{fig:edge_xray}(d) for cluster‐based SVD indicate that both methods perform comparably when the compression factor is low, with only subtle differences in edge clarity. Nonetheless, cluster‐based SVD still retains slightly sharper and more accurate edge details. This is further quantified by the edge loss maps in Fig.~\ref{fig:edge_xray}(e) and Fig.~\ref{fig:edge_xray}(f). The difference is also captured numerically by the EPI, where global SVD achieves 0.8579, while cluster‐based SVD outperforms it with a value of 0.9375.
\begin{figure}
    \centering
    \includegraphics[width=1.0\linewidth]{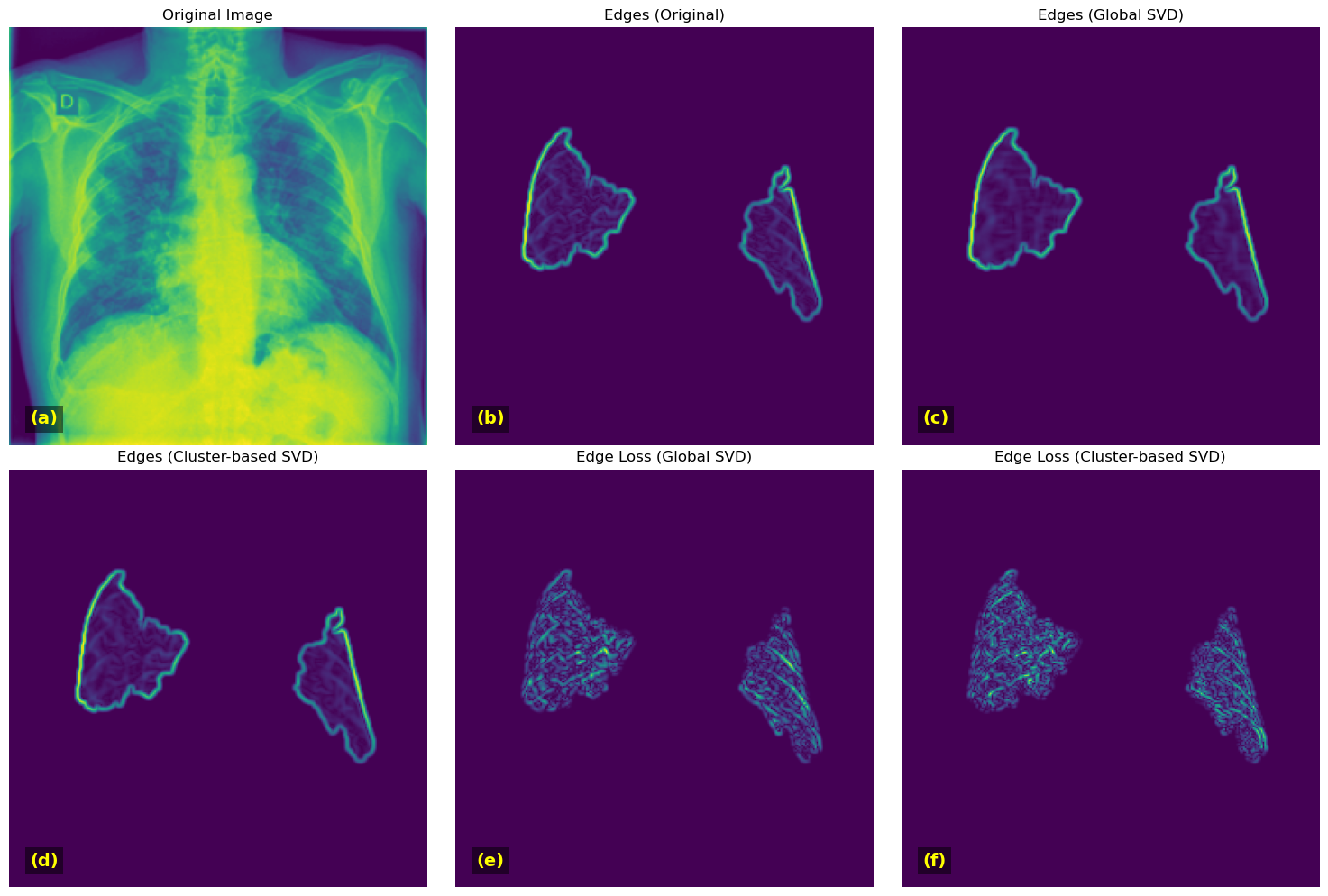}
    \caption{The figure compares the effects of global SVD and cluster-based SVD on edge preservation in a chest X-ray image. (a) shows the original X-ray. (b) presents the detected edges in the original image. (c) and (d) display the edges in the globally compressed and cluster-based compressed images, respectively. (e) and (f) show the edge loss maps for global SVD and cluster-based SVD, respectively.}
    \label{fig:edge_xray}
\end{figure}

While Figs.~\ref{fig:infection_xray} and ~\ref{fig:noninfection_xray} evaluate the performance of the two methods across varying compression factors. Fig.~\ref{fig:infection_xray} examines the infected (COVID-19) region while Fig.~\ref{fig:noninfection_xray} focuses on the non-infected region. In both cases, cluster‐based SVD consistently outperforms global SVD by achieving higher PSNR, IoU and SSIM values along with lower MSE.
\begin{figure*}
    \centering
    \includegraphics[width=1.0\linewidth]{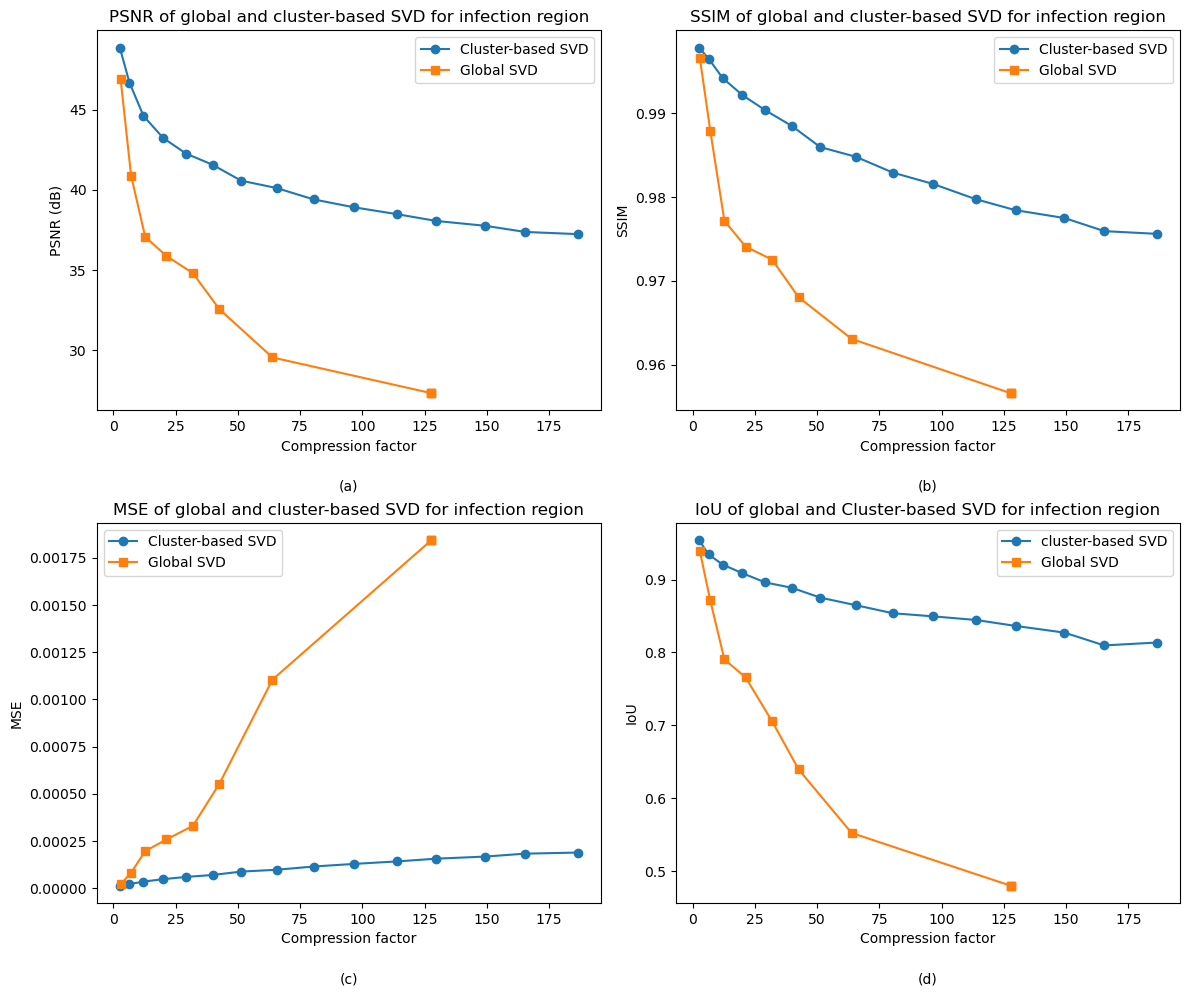}
    \caption{The figure compares global SVD and cluster-based SVD in the infection region of a chest X-ray, evaluated across varying compression factors. (a) PSNR (dB), (b) SSIM, (c) MSE, and (d) IoU.}
    \label{fig:infection_xray}
\end{figure*}
\begin{figure*}
    \centering
    \includegraphics[width=1.0\linewidth]{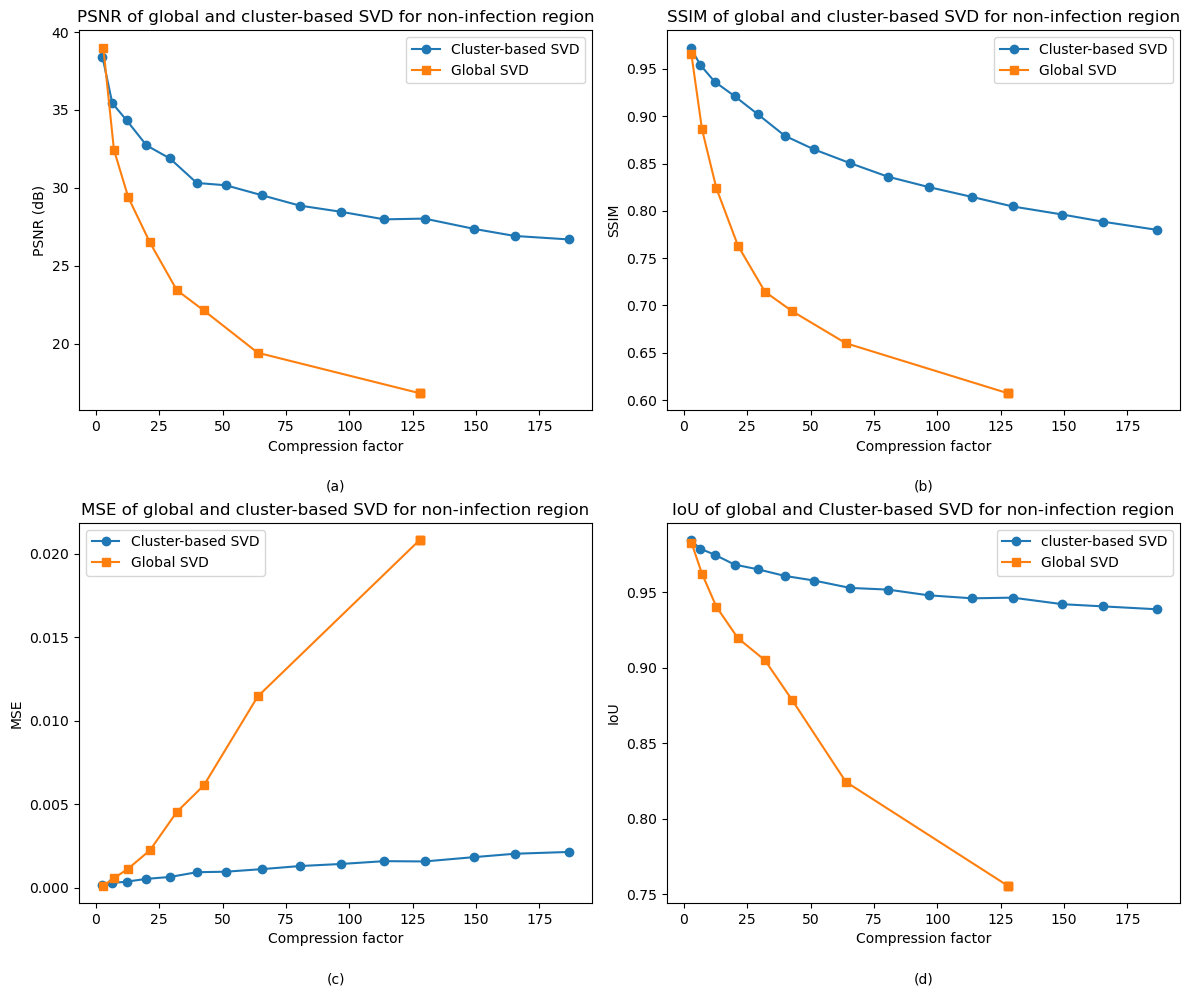}
    \caption{The figure compares global SVD and cluster-based SVD in the non-infection region of a chest X-ray, evaluated across varying compression factors. (a) PSNR (dB), (b) SSIM, (c) MSE, and (d) IoU.}
    \label{fig:noninfection_xray}
\end{figure*}

\end{document}